\documentclass[nohyperref]{article}

\usepackage{microtype}
\usepackage{bibentry,color,epsfig,framed,graphics}
\usepackage{graphicx}
\usepackage{subfigure}
\usepackage{booktabs} 
\usepackage{enumitem}
\usepackage{mycommands2}
\usepackage{wrapfig}
\usepackage{hyperref}

\usepackage[accepted]{icml2022}

\usepackage{amsmath}
\usepackage{amssymb}
\usepackage{mathtools}
\usepackage{amsthm}

\DeclareMathOperator*{\argmin}{arg\,min}

\DeclareMathOperator*{\supp}{supp}
\DeclareMathOperator*{\Th}{th}

\theoremstyle{plain}
\newtheorem{theorem}{Theorem}[section]

\newtheorem{lemma}[theorem]{Lemma}
\newtheorem{corollary}[theorem]{Corollary}
\theoremstyle{definition}
\newtheorem{definition}[theorem]{Definition}

\theoremstyle{remark}
\newtheorem{remark}[theorem]{Remark}

\usepackage[textsize=tiny]{todonotes}

\icmltitlerunning{Feature Selection using e-values}

\begin{document}

\twocolumn[
\icmltitle{Feature Selection using e-values}




\begin{icmlauthorlist}
\icmlauthor{Subhabrata Majumdar}{umn,spl}
\icmlauthor{Snigdhansu Chatterjee}{umn}
\end{icmlauthorlist}

\icmlaffiliation{umn}{School of Statistics, University of Minnesota Twin Cities, Minneapolis, MN, USA}
\icmlaffiliation{spl}{Currently at Splunk}

\icmlcorrespondingauthor{Subhabrata Majumdar}{smajumdar@splunk.com}

\icmlkeywords{Feture selection, model selection, data depth, bootstrap}

\vskip 0.3in
]



\printAffiliationsAndNotice{} 

\begin{abstract}
In the context of supervised parametric models, we introduce the concept of {\it e-values}. An e-value is a scalar quantity that represents the proximity of the sampling distribution of parameter estimates in a model trained on a subset of features to that of the model trained on all features (i.e. the full model). Under general conditions, a rank ordering of e-values separates models that contain all essential features from those that do not.

The e-values are applicable to a wide range of parametric models. We use data depths and a fast resampling-based algorithm to implement a feature selection procedure using e-values, providing consistency results. For a $p$-dimensional feature space, this procedure requires fitting only the full model and evaluating $p+1$ models, as opposed to the traditional requirement of fitting and evaluating $2^p$ models. Through experiments across several model settings and synthetic and real datasets, we establish that the e-values method as a promising general alternative to existing model-specific methods of feature selection.
\end{abstract}

\section{Introduction}\label{sec:IntroSection}
\label{Section:Introduction}


In the era of big data, feature selection in supervised statistical and machine learning (ML) models  helps cut through the noise of superfluous features, provides storage and computational benefits, and gives model intepretability. Model-based feature selection can be divided into two categories \cite{Guyon03}: wrapper methods that evaluate models trained on multiple feature sets \cite{ref:Schwarz_AoS78461,ref:Shao_JASA96655}, and embedded methods that combine feature selection and training, often through sparse regularization \cite{Tibshirani96,Zou06,elasticnet}. Both categories are extremely well-studied for independent data models, but have their own challenges. Navigating an exponentially growing feature space using wrapper methods is NP-hard \citep{Natarajan95}, and case-specific search strategies like $k$-greedy, branch-and-bound, simulated annealing are needed. Sparse penalized embedded methods can tackle high-dimensional data, but have inferential and algorithmic issues, such as biased Lasso estimates \citep{ZhangZhang14} and the use of convex relaxations to compute approximate local solutions \citep{WangKimLi13,ZouLi08}.

Feature selection in dependent data models has received comparatively lesser attention. Existing implementations of wrapper and embedded methods have been adapted for dependent data scenarios, such as mixed effect models \cite{MezaLahiri05,NguyenJiang14,ref:PengLu_JMVA12109} and spatial models \cite{HuangEtal10,LeeGhosh09}. However these suffer from the same issues as their independent data counterparts. If anything, the higher computational overhead for model training makes implementation of wrapper methods even harder in this situation!


In this paper, we propose the framework of {\it e-values} as a common principle to perform best subset feature selection in a range of parametric models covering independent and dependent data settings. In essence ours is a wrapper method, although it is able to determine important features affecting the response {\it by training a single model}: the one with all features, or the full model. We achieve this by utilizing the information encoded in the {\it distribution} of model parameter estimates. Notwithstanding recent efforts \cite{LaiHannigLee15,SinghEtal07}, parameter distributions that are fully data-driven have generally been underutilized in data science. In our work, e-values score a candidate model to quantify the similarity between sampling distributions of parameters in that model and the full model. Sampling distribution is the distribution of a parameter estimate, based on the random data samples the estimate is calculated from. We access sampling distributions using the efficient Generalized Bootstrap technique \cite{ref:CBose_AoS05414}, and utilize them using data depths, which are essentially point-to-distribution inverse distance functions \citep{tukey75,zuo03, ref:ZuoSerfling_AoS00461}, to compute e-values.

\paragraph{How e-values perform feature selection}
Data depth functions quantify the inlyingness of a point in multivariate space with respect to a probability distribution or a multivariate point cloud, and have seen diverse applications in the literature \citep{Jornsten04,NarisettyNair16,regdepth}. As an example, consider the {\it Mahalanobis depth}, defined as
$$ D(x, F) = [1 + (x-\mu_F)^\top \Sigma_F^{-1} (x-\mu_F)]^{-1}, $$
for $x \in \mathbb R^p$ and $F$ a $p$-dimensional probability distribution withg mean $\mu_F$ and covariance matrix $\Sigma_F$. When $x$ is close to $\mu_F$, $D(x, F)$ takes a high value close to 1. On the other hand, as $\| x\| \rightarrow \infty$, the depth at $x$ approaches 0. Thus, $D(x, F)$ quantifies the proximity of $x$ to $F$. Points with high depth are situated 'deep inside' the probability distribution $F$, while low depth points sit at the periphery of $F$.

Given a depth function $D$ we define e-value as the mean depth of the finite-sample estimate of a parameter $\beta$ for a candidate model $\cM$, with respect to the sampling distribution of the full model estimate $\hat \beta$:
$$ e (\cM) = \BE D( \hat \beta_\cM, [ \hat \beta]),$$
where $\hat \beta_\cM$ is the estimate of $\beta$ assuming model $\cM$, $[\hat \beta]$ is the sampling distribution of $\hat \beta$, and $\BE(\cdot)$ denotes expectation. For large sample size $n$, the index set $\cS_{select}$ obtained by Algorithm~\ref{alg:alg1} below elicits all non-zero features in the true parameter. We use bootstrap to obtain multiple copies of $\hat \beta$ and $\hat\bfbeta_{-j}$ for calculating the e-values in steps 1 and 5.
\begin{algorithm}
\label{alg:alg1}
\caption{Best subset selection using e-values}
\begin{enumerate}[nolistsep,leftmargin=*]
    \item[1:] Obtain full model e-value: $e(\cM_{full}) = \BE D ( \hat\bfbeta, [ \hat \bfbeta] )$.
    \item[2:] Set $\cS_{select} = \phi$.
    \item[3:] For $j$ in $1:p$
    \item[4:] \hspace{1em} Replace $j^\text{th}$ index of $\hat\bfbeta$ by 0, name it $\hat \beta_{-j}$.
    \item[5:] \hspace{1em} Obtain $e(\cM_{-j}) = \BE D ( \hat\bfbeta_{-j}, [ \hat \bfbeta] )$.
    \item[6:] \hspace{1em} If $e(\cM_{-j}) < e(\cM_{full}))$
    \item[7:] \hspace{2em} Set $ \cS_{select} \leftarrow \{ \cS_{select}, j \} $.
\end{enumerate}
\end{algorithm}
\begin{figure}[t!]
\centering
\includegraphics[width=.5\columnwidth]{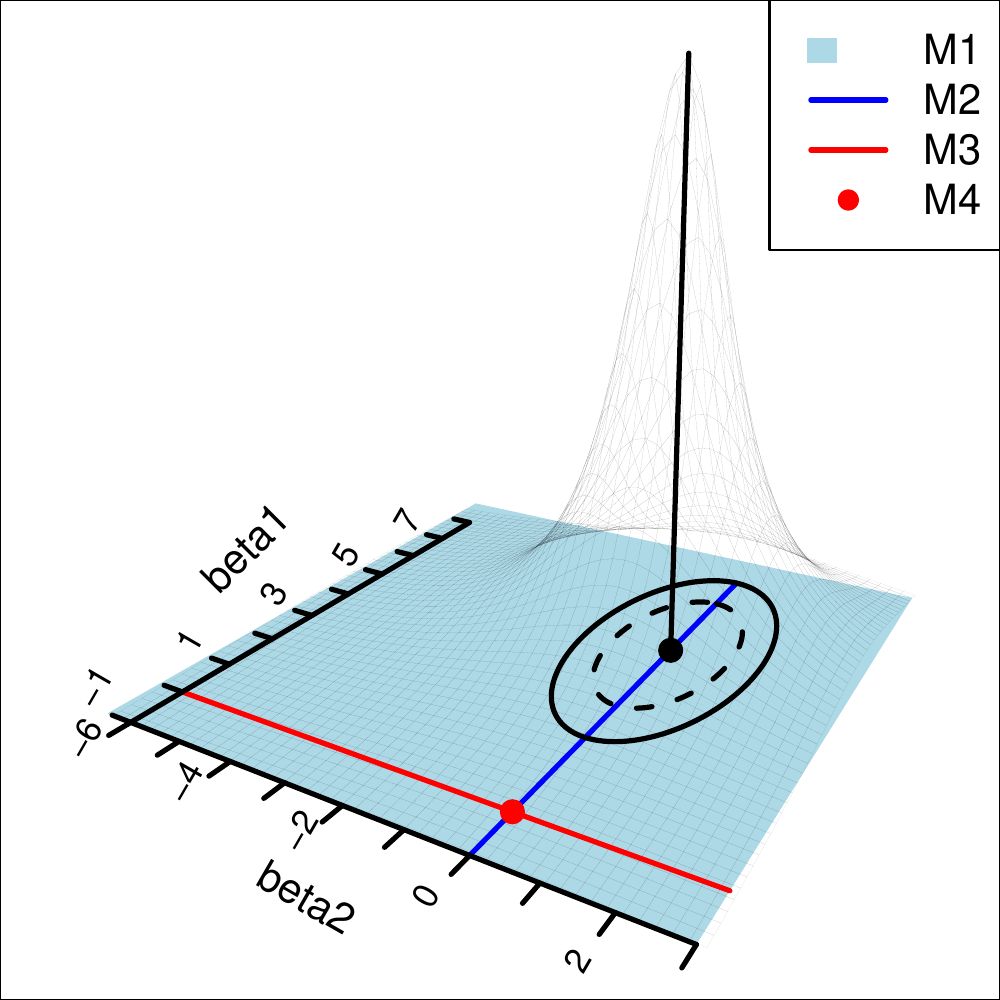}
\caption{The e-value method. Solid and dotted ellipses represent Mahalanobis depth contours at some $\alpha >0$ for sample sizes $n_1, n_2; n_2 > n_1$.}
\label{fig:exampleFig}
\end{figure}
%

As an example, consider a linear regression with two features, Gaussian errors $\epsilon \sim N(0,1)$, and the following candidate models (Figure~\ref{fig:exampleFig}). We denote by $\Theta_i$ the domain of parameters considered in model $\cM_i; i = 1,\ldots,4$.

\begin{center}
\begin{tabular}{ll}
$\cM_1: \ \ Y = X_1 \beta_1 + X_2 \beta_2 + \epsilon$; &$\ \ \Theta_1 = \BR^2,$\\
$\cM_{2}: \ \ Y  =  X_1 \beta_1 + \epsilon$; &$\ \ \Theta_2 = \BR \times \{ 0 \},$ \\
$\cM_{3}: \ \ Y  =  X_2 \beta_2 + \epsilon$; &$\ \ \Theta_3 = \{ 0 \} \times \BR,$\\
$\cM_{4}: \ \ Y  =  \epsilon $; &$\ \ \Theta_4 = (0,0)^T.$
\end{tabular}
\end{center}
Let $\beta_0 = (5,0)^T$ be the true parameter. The full model sampling distribution, $\hat \beta \sim \mathcal N(\beta_0, (X^T X)^{-1})$, is more concentrated around $\beta_0$ for higher sample sizes. Thus the depths at points along the (red) line $\beta_2=0$, and the origin (red point), become smaller, and mean depths approach 0 for $\cM_3$ and $\cM_4$. On the other hand, since depth functions are affine invariant \cite{ref:ZuoSerfling_AoS00461}, mean depths calculated over parameter spaces for $\cM_1$ (blue line) and $\cM_2$ (blue surface) remain the same, and do not vanish as $n \rightarrow \infty$. Thus, e-values of the `good' models $\cM_1, \cM_2$---models the parameter spaces of which contain $\beta_0$---separate from those of the `bad' models $\cM_3, \cM_4$ more and more as $n$ grows. Algorithm~\ref{alg:alg1} is the result of this separation, and a rank ordering of the `good' models based on how parsimonious they are (Theorem~\ref{Theorem:ThmRightNested}).

\section{Related work}
\label{sec:related}
Feature selection is a widely studied area in statistics and ML. The vast amount of literature on this topic includes classics like the Bayesian Information Criterion \citep[BIC]{ref:Schwarz_AoS78461} or Lasso \cite{Tibshirani96}, and recent advances such as Mixed Integer Optimization \citep[MIO]{BertsimasEtal16}. To deal with the increasing scale and complexity of data in recent applications, newer streams of work have also materialized---such as nonlinear and model-independent methods \cite{song12a}, model averaging \cite{FragosoEtal18}, and dimension reduction techniques \cite{MaZhu13}. We refer the reader to \citet{Guyon03} and the literature review in \citet{BertsimasEtal16} for a broader overview of feature selection methods, and focus on the three lines of work most related to our formulation.

\citet{ref:Shao_JASA96655} first proposed using bootstrap for feature selection,  with the squared residual averaged over a large number of resampled parameter estimates from a $m$-out-of-$n$ bootstrap as selection criterion. 
Leave-One-Covariate-Out \citep[LOCO]{loco} is based on the idea of sample-splitting. The full model and leave-one-covariate-out models are trained using one part of the data, and the rest of the data are used to calculate the importance of a feature as median difference in predictive performance between the full model and a LOCO model. 
Finally, \citet{knockoff15} introduced the powerful idea of Knockoff filters, where a `knockoff' version of the original input dataset imitating its correlation structure is constructed. 
This method is explicitly able to control False Discovery Rate.

Our framework of e-values has some similarities with the above methods, such as the use of bootstrap (similar to \citet{ref:Shao_JASA96655}), feature-specific statistics (similar to LOCO and Knockoffs), and evaluation at $p+1$ models (similar to LOCO). However, e-values are also significantly different. They do not suffer from the high computational burden of model refitting over multiple bootstrap samples (unlike \citet{ref:Shao_JASA96655}), are not conditional on data splitting (unlike LOCO), and have a very general theoretical formulation (unlike Knockoffs). Most importantly, unlike all the three methods discussed above, e-values require fitting only a single model, and work for dependent data models.

Previously, \citet{VanderweeleDing17} have used the term 'E-Value' in the context of sensitivity analysis. However, our and their definition of e-values are quite different. We see our e-values as a means to {\it evaluate} a feature with reference to a model. There are some parallels to the well known $p$-values used for hypothesis testing, but we see the e-value as a more general quantity that covers dependent data situations, as well as general estimation and hypothesis testing problems. While the current paper is an application of e-values for feature selection, we plan to build up on the framework introduced here on other applications, including group feature selection, hypothesis testing, and multiple testing problems.
\section{Preliminaries}
\label{sec:PrelimSection}

For a positive integer $n \geq 1$, Let $\cZ_n = \{ Z_1, \ldots, Z_n \}$ be an observable array of random variables that are not necessarily independent. We parametrize $\cZ_n$ using a parameter vector $\theta \in \Theta \subseteq \BR^{p}$, and energy functions $\{ \psi_{i} (\theta, Z_{i}): 1 \leq i \leq n \}$. 
%
%
We assume that there is a true unknown vector of parameters $\theta_{0}$, which is the unique minimizer of
%
$
\Psi_n (\theta) = \BE \sum_{i=1}^n \psi_{i} (\theta, Z_{i})
$. 
%

\paragraph{Models and their characterization}
Let $\cS_{*} = \cup_{ \theta \in \Theta} \supp(\theta)$ be the common non-zero support of all estimable parameters $\theta$ (denoted by $\supp(\theta)$). In this setup, we associate a model $\cM$ with two quantities (a) The set of indices $\cS \subseteq \cS_{*} $ with unknown and estimable parameter values, and (b) an ordered vector of known constants $C = (C_{j}: j \notin \cS)$ at indices not in $\cS$. 
%
%
Thus a generic parameter vector for the model $\cM \equiv (\cS, C)$, denoted by $\theta_{m} \in \Theta_{m} \subseteq \Theta = \prod_j \Theta_{j}$, consists of unknown indices and known constants
\baq 
 \vectheta_{m j} = \left\{ \begin{array}{ll}
 \text{ unknown} \ \theta_{ m j} \in \matTheta_{j} & \text{ for } 
 			j \in \cS, \\
 \text{ known} \  C_{j} \in \matTheta_{j} & \text{ for } j \notin \cS.
\end{array}
\right.
\label{eq:vectheta_mn}
\eaq

Given the above formulation, we characterize a model.

\begin{definition}
A model $\cM$ is called \textit{adequate} if $\sum_{j \notin \cS} | C_{j}  - {\theta}_{0 j} | = 0$. A model that is not adequate, is called an {\textit{inadequate}} model.
\end{definition}

By definition the full model is always adequate, as is the model corresponding to the singleton set containing the true parameter
. Thus the set of adequate models is non-empty by construction. 

Another important notion is the one of \textit{nested models}.
\begin{definition}
We consider a model $\cM_{1}$ to be nested in $\cM_{2}$, notationally $\cM_{1} \prec \cM_{2}$, if 
$\cS_{1} \subset \cS_{2}$ and $C_{2}$ is a subvector of $C_{1}$.
\end{definition}
\noindent If a model is adequate, then any model it is nested in is also adequate. In the context of feature selection this obtains a rank ordering: the most parsimonious adequate model, with $\cS = \supp (\beta_0)$ and $C = 0_{p - | \cS |}$, is nested in all other adequate models. All models nested {\it within} it are inadequate.

\paragraph{Estimators}
The estimator $\hat{\theta}_{*}$ of $\theta_{0}$ is obtained as a minimizer of the sample analog of $\Psi_n (\theta)$:
\baq
\hat{\theta}_{*} = 
\argmin_{\theta} {\Psi}_n (\theta) =
\argmin_{\theta} \sum_{i = 1}^{n} \psi_{i}  \bigl( \theta, Z_{i} \bigr).
\label{eq:Psistarnhat}
\eaq
%
Under standard regularity conditions on the energy functions, $a_n (\hat \theta_* - \theta_*)$ converges to a $p$-dimensional Gaussian distribution as $n \rightarrow \infty$, where $a_n \uparrow \infty$ is a sequence of positive real numbers (Appendix~\ref{app:boot}).

\begin{remark}
For generic estimation methods based on likelihood and estimating equations, the above holds with $a_{n} \equiv {n}^{1/2}$, resulting in the standard `root-$n$' asymptotics.
\end{remark}





The estimator in \eqref{eq:Psistarnhat} corresponds to the {\it full model} $\cM_{*}$, i.e. the model where all indices in $\cS_{*}$ are estimated. For any other model $\cM$,
%
%
we simply augment entries of $\hat \theta_{*}$ at indices in $\cS$ with elements of $C$ elsewhere to obtain a model-specific coefficient estimate $\hat\theta_m$:
\baq 
 \hat{\vectheta}_{m j} = \left\{ \begin{array}{ll}
 \hat{\theta}_{ * j} & \text{ for } j \in \cS, \\
 C_{j} & \text{ for } j \notin \cS.
\end{array}
\right.
\label{eq:thetahat_mn}
\eaq
The logic behind this plug-in estimate is simple: for a candidate model $\cM$, a joint distribution of its 
estimated parameters, i.e. $[\hat{\vectheta}_{\cS}]$, can actually be obtained from $[ \hat{\vectheta}_{*}] $ by taking its marginals at indices $\cS$.




\paragraph{Depth functions}
Data depth functions \citep{ref:ZuoSerfling_AoS00461} quantify the closeness of a point in multivariate space to a probability distribution or data cloud. Formally, let $\cG$ denote the set of non-degenerate probability measures on $\BR^{p}$. We consider $D : \BR^{p} \times \cG \raro [0, \infty)$ to be a data depth function if it satisfies the following properties:
\begin{enumerate}[nolistsep]
\item[(B1)]
{\it Affine invariance}: For any non-singular matrix $\bfA \in \BR^{p \times p}$, and $b \in \BR^p$ and random variable $Y$ having distribution $\BG \in \cG$, 
$
D (x, \BG) = D ( \bfA x + b, [ \bfA Y + b ])
$.
\item[(B2)]
{\it Lipschitz continuity}: For any $\BG \in \cG$, there exists $\delta > 0$ and $\alpha \in (0, 1)$, possibly depending on $\BG$ such that whenever  $| x - y | < \delta$, we have 
$ 
| D (x, \BG) - D (y, \BG) | < | x - y|^{\alpha} 
$.

\item[(B3)]
Consider random variables $Y_{n} \in \BR^{p}$ such that $[ Y_n] \leadsto \BY \in \cG$. Then $D (y, [Y_n])$  converges uniformly to $D (y, \BY) $. So that, if $Y \sim \BY$, then 
$
\lim_{n \raro \infty} \BE D (Y_n, [Y_n]) = \BE D (Y, \BY) < \infty
$.

\item[(B4)]
For any $\BG \in \cG$, $\lim_{\| \bfx \| \raro \infty} D( \bfx, \BG) = 0$.

\item[(B5)]
For any  $\BG \in \cG$ with a point of symmetry $\bfmu (\BG) \in \BR^{p}$, $D(.,\BG)$ is maximum at $\bfmu (\BG)$:
$$ D (\bfmu (\BG), \BG) =
\sup_{\bfx \in  \BR^{ p}} D ( \bfx, \BG) < \infty. $$
Depth decreases along any ray between $\bfmu (\BG)$ to $x$, i.e. for $ t \in (0, 1), \bfx \in  \BR^{p}$,
\begin{align*}
D ( \bfx, \BG) &< D (\bfmu (\BG) + t (\bfx - \bfmu (\BG)), \BG) \\
&< D (\bfmu (\BG), \BG).
\end{align*}

\end{enumerate}
Conditions (B1), (B4) and (B5) are integral to the formal definition of data depth \citep{ref:ZuoSerfling_AoS00461}, while (B2) and (B3) implicitly arise for several depth functions \citep{MoslerChapter13}. We require only a subset of (B1)-(B5) for the theory in this paper, but use data depths throughout for simplicity.
\section{The e-values framework}
\label{sec:evalueSection}
The {\it e-value} of model $\cM$ is the mean depth of $\hat \theta_{\cM}$ with respect to $[ \hat \theta ]$: $e_n (\cM) = \BE  D (  \hat \theta_{ m}, [\hat \theta_{*}] )$. 

\subsection{Resampling approximation of e-values}
\label{sec:BootSection}
Typically, the distributions of either of the random quantities $\hat \theta_{m}$ and $\hat \theta_{*}$, are unknown, and have to be elicited from the data. Because of the plugin method in \eqref{eq:thetahat_mn}, only $[ \hat \theta_{*}]$ needs to be approximated. We do this using Generalized Bootstrap \citep[GBS]{ref:CBose_AoS05414}. For an exchangeable array of  non-negative random variables independent of the data as resampling weights: $\cW_{r} = ( \BW_{r 1}, \ldots,  \BW_{r n})^{T} \in \BR^{n}$, the GBS estimator $\hat{\vectheta}_{r *}$ is the minimizer of 
\baq
{\Psi}_{r n} (\theta) 
= \sum_{i = 1}^{n} \BW_{r i}\psi_{i}  \bigl( \theta, Z_{i} \bigr).
\label{eq:Psisnhat_R}
\eaq

We assume the following conditions on the resampling weights and their interactions as $n \raro \infty$:
\baq
& \BE \BW_{r 1} = 1, 
\BV \BW_{r 1} = \tau_{n}^{2} = o ( a_{n}^{2}) \uparrow \infty,
\BE W_{r 1}^{4} < \infty,\notag\\
& \BE W_{r 1} W_{r 2} = O (n^{-1}), 
\BE W_{r 1}^{2} W_{r 2}^{2} \raro 1.
\label{eq:W_cond}
\eaq
%
%
Many resampling schemes can be described as GBS, such as the $m$-out-of-$n$ bootstrap \cite{BickelSakov08} and scale-enhanced wild bootstrap \cite{chatt19}. Under fairly weak regularity conditions on the first two derivatives $\psi'_{i}$ and $\psi''_{i}$ of $\psi_{i}$, $(a_n/\tau_n) (\hat \theta_{r *} - \hat \theta_{*})$ and $a_n (\hat \theta_{*} - \theta_{0})$ converge to the same weak limit in probability (See Appendix~\ref{app:boot}).

Instead of repeatedly solving \eqref{eq:Psisnhat_R}, we use model quantities computed while calculating $\hat \theta_{*}$ to obtain a first-order approximation of $\hat \theta_{r *}$. 
\baq\label{eqn:BootEqn}
\hat\theta_{r *} & = \hat\bftheta_{*} - \frac{ \tau_n}{a_n} \left[ \sum_{i=1}^n \psi_i'' (\hat \bftheta_{*}, Z_i) \right]^{-1/2}
\sum_{i=1}^{n} W_{r i} \psi_i' (\hat \bftheta_{*}, Z_i) \notag\\
\quad &+ \bfR_{r},
\eaq
where $\BE_r \| \bfR_{r} \|^2 = o_P(1)$, and $W_{r i} = (\BW_{r i} - 1)/\tau_n$. Thus only Monte Carlo sampling is required to obtain the resamples. Being an embarrassingly parallel procedure, this results in significant computational speed gains.

To estimate $e_n(\cM)$ we obtain two independent sets of weights $\{ \cW_{r}; r = 1, \ldots, R\}$ and $\{ \cW_{r_1}; r_1 = 1, \ldots, R_1 \}$ for large integers $R, R_1$. We use the first set of resamples to obtain the distribution $[\hat{\vectheta}_{r *}]$ to approximate $[ \hat \theta_{*} ]$, and the second set of resamples to get the plugin estimate $\hat{\vectheta}_{r_{1} m}$:
\begin{align}\label{eqn:bootEstEqn}
 \hat{\theta}_{r_{1} m j} = \left\{ \begin{array}{ll}
 	 \hat{\theta}_{r_1 * j} & \text{ for } 
 			j \in \cS, \\
 	 C_{j} & \text{ for } j \notin \cS.
\end{array}
\right.
\end{align}
Finally, the resampling estimate of a model e-value is: $
e_{r n} (\cM) =
\BE_{r_{1}} D \bigl(  \hat{\vectheta}_{r_{1} m},
[\hat{\vectheta}_{r *}] \bigr) 
$, 
%
%
where $\BE_{r_{1}}$ is expectation, conditional on the data, computed using the resampling $r_{1}$.

\subsection{Fast algorithm for best subset selection}
\label{subsec:Algo}
For best subset selection, we restrict to the model class 
$
\mathbb M_{0} = \{ \cM: C_j=0 \quad \forall \quad j \notin \cS \}
$ that only allows zero constant terms. In this setup our fast selection algorithm consists of only three stages: (a) fit the full model and estimate its e-value, (b) replace each covariate by 0 and compute  e-value of all such reduced models, and (c) collect covariates dropping which causes the e-value to go down.
%

To fit the full model, we need to determine the estimable index set $\cS_*$. When $n > p$, the choice is simple: $\cS_* = \{ 1, \ldots, p\}$. When $p > n$, we need to ensure that $p' = | \cS_* | < n$, so that $\hat \theta_{*}$ (properly centered and scaled) has a unique asymptotic distribution. Similar to past works on feature selection \cite{LaiHannigLee15}, we use a feature screening step before model fitting to achieve this (Section~\ref{sec:Theory}).

After obtaining $\cS_*$ and $\hat \theta_*$, for each of the $p'+1$ models under consideration: the full model and all drop-1-feature models, we follow the recipe in Section \ref{sec:BootSection} to get bootstrap e-values. This gives us all the components for a sample version of Algorithm \ref{alg:alg1}, which we present as Algorithm~\ref{alg:algoselectboot}.

\begin{algorithm}[H]
\caption{Best subset feature selection using e-values and GBS}
\label{alg:algoselectboot}
\noindent 1. Fix resampling standard deviation $\tau_n$.

\noindent 2. Obtain GBS samples: $\cT = \{ \hat\theta_{1 *}, ..., \hat\theta_{R *} \} $, and $\cT_1 = \{ \hat\theta_{1 *}, ..., \hat\theta_{R_1 *} \} $ using \eqref{eqn:BootEqn}.

\noindent 3. Calculate $\hat e_{r n} (\cM_*) = \frac{1}{R_1} \sum_{r_1=1}^{R_1} D ( \hat\theta_{r_1 *}, [\cT] )$.

\noindent 4. Set $\hat \cS_0 = \phi$.

\noindent 5. {\bf for} $j$ in $1:p$

\noindent 6. \hspace{1em} {\bf for} $r_1$ in $1:R_1$

\noindent 7. \hspace{2em} Replace $j^\text{th}$ index of $\hat\theta_{* r_1}$ by 0 to get $\hat \theta_{r_1, -j}$.

\noindent 8. \hspace{1em} Calculate $\hat e_{r n} (\cM_{-j}) = \frac{1}{R_1} \sum_{r_1=1}^{R_1} D ( \hat\theta_{r_1, -j}, [\cT] )$.

\noindent 9. \hspace{1em} {\bf if} $\hat e_{r n} (\cM_{-j}) < \hat e_{r n} (\cM_*)$

\noindent 10. \hspace{2em} Set $ \hat \cS_0 \leftarrow \{ \hat \cS_0, j \} $.
\end{algorithm}

\paragraph{Choice of $\tau_n$}
The intermediate rate of divergence for the bootstrap standard deviation $\tau_n$: $\tau_n \rightarrow \infty, \tau_n/a_n \rightarrow 0$, is a necessary and sufficient condition for the consistency of GBS \citep{ref:CBose_AoS05414}, and that of the bootstrap approximation of population e-values (Theorem~\ref{thm:BootConsistency}). 
We select $\tau_n$ using the following quantity, which we call Generalized Bootstrap Information Criterion (GBIC):
\begin{align}\label{eqn:biceqn}
\text{GBIC}(\tau_n) &= \sum_{i=1}^n \psi_{i} \left( \hat \theta(\hat \cS_0,\tau_n), Z_{i} \right) + \notag\\
& \frac{\tau_n}{2} \left| \supp(\hat \theta(\hat \cS_0,\tau_n)) \right|,
\end{align}
where $\hat \theta(\hat \cS_0,\tau_n)$ is the refitted parameter vector using the index set $\hat \cS_0$ selected by running Algorithm~\ref{alg:algoselectboot} with $\tau_n$. We repeat this for a range of $\tau_n$ values, and choose the index set corresponding to the $\tau_n$ that gives the smallest $\text{GBIC}(\tau_n)$. For our synthetic data experiments we take $\tau_n = \log(n)$ and use GBIC, while for one real data example, we use validation on a test set to select the optimal $\tau_n$, both with favorable results.



\paragraph{Detection threshold for finite samples}
\label{subsec:threshold}
In practice---especially for smaller sample sizes---
due to sampling uncertainty it may be difficult to ensure that the full model e-value exactly separates the null and non-null e-values
, and small amounts of false positives or negatives may occur in the selected feature set $\hat \cS_0$. One way to tackle this is by shifting the detection threshold in Algorithm~\ref{alg:algoselectboot} by a small $\delta$:
$$ \hat e^\delta_{rn} (\cM_*) = (1+\delta) \hat e_{rn} (\cM_*). $$
To prioritize true positives or true negatives, $\delta$ can be set to be positive or negative respectively. 
In one of our experiments (Section~\ref{subsec:sim3}), setting $\delta = 0$ results in a small amount of false positives due to a few non-null features having e-values close to the full model e-value. Setting $\delta$ to small positive values gradually eliminates these false positives.


\section{Theoretical results}
\label{sec:Theory}

We now investigate theoretical properties of e-values. Our first result separates inadequate models from adequate models at the population level, and gives a rank ordering of adequate models using their population e-values.

\begin{theorem}
\label{Theorem:ThmRightNested}
Under conditions B1-B5, for a finite sequence of adequate models $\cM_{1} \prec  \ldots \prec \cM_{k}$ and any inadequate models $\cM_{k +1}, \ldots, \cM_{K}$, we have for large $n$
\ban 
 e_n(\cM_{1} ) > \ldots > e_n (\cM_{k})  > \max_{j \in \{ k + 1, \ldots K \}} e_n (\cM_{j}).
\ean
As $n \rightarrow \infty$, $e_n (\cM_i) \rightarrow \BE D(Y, [Y]) < \infty$ with $Y$ having an elliptic distribution if $i \leq K$, else $e_n (\cM_i) \rightarrow 0$.
\end{theorem}
%
We define the \textit{data generating model} as $\cM_0 \prec \cM_*$ with estimable indices $\cS_0 = \supp (\theta_0)$ and constants $C_0 = 0_{p - | \cS_0|}$. Then we have the following.
\begin{corollary}\label{Corollary:ZeroModelCorollary}
Assume the conditions of Theorem~\ref{Theorem:ThmRightNested}. Consider the restricted class of candidate models $\mathbb M_{0}$ in Section~\ref{subsec:Algo}, where all known coefficients are fixed at 0. Then for large enough $n$, $\cM_{0} = \arg \max_{ \cM \in \mathbb M_{0}} [ e_{n} ( \cM ) ]$.
%
\end{corollary}
Thus, when only the models with known parameters set at 0 (the model set $\mathbb M_0$) are considered, e-value indeed maximizes at the true model at the {\it population level}. However there are still $2^p$ possible models. This is where the advantage of using e-values---their one-step nature---comes through.

\begin{corollary}\label{Corollary:AlgoCorollary}
Assume the conditions of Corollary~\ref{Corollary:ZeroModelCorollary}. Consider the models $\cM_{ -j} \in \mathbb M_0$ with  $\cS_{-j} = \{ 1, \ldots ,p \} \setminus \{ j \}$ for $j = 1, \ldots ,p$. Then covariate $j$ is a necessary component of $\cM_0$, i.e. $\cM_{-j}$ is an inadequate model, if and only if for large $n$ we have $e_{n} (\cM_{-j}) < e_{n} (\cM_*)$.
%
%
\end{corollary}

Dropping an essential feature from the full model makes the model inadequate, which has very small e-value for large enough $n$, whereas dropping a non-essential feature increases the e-value (Theorem~\ref{Theorem:ThmRightNested}). Thus, {\it simply collecting features dropping which cause a decrease in the e-value suffices for feature selection}.



Following the above results, we establish model selection consistency of Algorithm~\ref{alg:algoselectboot} at the sample level. This means that the probability that the one-step procedure is able to select the true model feature set goes to 1, when the procedure is repeated for a large number of randomly drawn datasets from the data-generating distribution.
\begin{theorem}
\label{thm:BootConsistency}
Consider two sets of generalized bootstrap estimates of $\hat \theta_*$: $\cT = \{ \hat \theta_{r *}: r = 1, \ldots, R \}$ and $\cT_1 = \{ \hat \theta_{r_1 *}: r_1 = 1, \ldots, R_1 \}$. Obtain sample e-values as:
\baq\label{eqn:BootEvalueMC}
\hat e_{r n} (\cM) & = \frac{1}{R_1} \sum_{r_1=1}^{R_1} D( \hat \theta_{r_1 m}, [ \cT ]),
\eaq
where $[\cT]$ is the empirical distribution of the corresponding bootstrap samples, and $\hat \theta_{r_1 m}$ are obtained as in Section~\ref{sec:BootSection}. Consider the feature set $\hat \cS_0 = \{ \hat e_{r n} (\cM_{-j} ) < \hat e_{r n} (\cM_*) \}$. Then as $n, R, R_1 \rightarrow \infty$, $P_2( \hat \cS_0 = \cS_0 ) \rightarrow 1$, with $P_2$ as probability conditional on data and bootstrap samples.
%
%
\end{theorem}
Theorem~\ref{thm:BootConsistency} is contingent on the fact that the the true model $\cM_0$ is a member of $\BM_0$, that is $\cS_0 \subseteq \cS_*$. This is ensured trivially when $n \geq p $. If $p > n$, $\BM_0$ is the set of all possible models on the feature set selected by an appropriate screening procedure. In high-dimensional linear models, we use Sure Independent Screening \citep[SIS]{FanLv08} for this purpose. Given that $\cS_*$ is selected using SIS, \citet{FanLv08} proved 
that, for constants $C>0$ and $\kappa$ that depend on the minimum signal in $\theta_0$,
\baq\label{eqn:SISConsistency}
P( \cM_0 \in \BM_0) \geq 1 - O \left( \frac{\exp[ -C n^{1-2 \kappa} ]}{\log n} \right).
\eaq
For more complex models, model-free filter methods \cite{KollerSahami96,ZhuEtal11} can be used to obtain $\cS_*$. For example, under mild conditions on the design matrix, the method of \citet{ZhuEtal11} is consistent:
\baq\label{eqn:ZhuConsistency}
P( | \cS_* \cap \cS_0^c | \geq r ) \leq \left( 1 - \frac{r}{p+d} \right)^d,
\eaq
with positive integers $r$ and $d$: $d = p$ being a good practical choice \citep[Theorem 3]{ZhuEtal11}. Combining \eqref{eqn:SISConsistency} or \eqref{eqn:ZhuConsistency} with Corollary~\ref{thm:BootConsistency} as needed establishes asymptotically accurate recovery of 
$\cS_0$ through Algorithm~\ref{alg:algoselectboot}.

\section{Numerical experiments}
\label{Section:Simulation}
We implement e-values using a GBS with scaled resample weights $W_{r i} \sim \text{Gamma}(1, 1)-1$, and resample sizes $R = R_{1} = 1000$. We use Mahalanobis depth for all depth calculations. Mahalanobis depth is much less computation-intensive than other depth functions \citep{DyckMoz16,LiuZuo14}, but is not usually preferred in applications due to its non-robustness. However, we do not use any robustness properties of data depth, so are able to use it without any concern. For each replication for each data setting and method, we compute performance metrics on test datasets of the same dimensions as the respective training dataset. All our results are based on 1000 such replications.

\subsection{Feature selection in linear regression}
\label{subsec:sim1}
Given a true coefficient vector $\beta_0$, we use the model $Z \equiv (Y,X), Y = X \beta_0 + \epsilon$, with $\epsilon \sim \cN_n(0, \sigma^2 I_n)$, and $n=500$ and $p = 100$. We generate the rows of $X$ independently from $\cN_p(0, \Sigma_X)$, where $\Sigma_X$ follows an equicorrelation structure having correlation coefficient $\rho$: $(\Sigma_X)_{ij} = \rho^{|i-j|}$. Under this basic setup, we consider the following settings to evaluate the effect of different magnitudes of feature correlation in $X$, sparsity level in $\beta_0$ and error variance $\sigma$.

\begin{itemize}[nolistsep,leftmargin=*]
\item{Setting 1 (effect of feature correlation):} We repeat the above setup for $\rho = 0.1, \ldots, 0.9$, fixing $\beta_0 = (1.5,0.5,1,1.5,1, 0_{p-5})$, $\sigma = 1$;
\item{Setting 2 (effect of sparsity level):} We consider $\beta_0 = (1_k, 0_{p-k})$, with varying degrees of the sparsity level $k = 5,10,15,20,25$. We fix $\rho = 0.5, \sigma = 1$;
\item{Setting 3 (effect of noise level):} We consider different values of noise level (equivalent to having different signal-to-noise ratios) by testing for $\sigma = 0.3, 0.5, \ldots, 2.3$, fixing $\beta_0 = (1_5, 0_{p-5}), \rho = 0.5$.
\end{itemize}

\begin{figure}[t!]
\centering
\includegraphics[width=1\columnwidth]{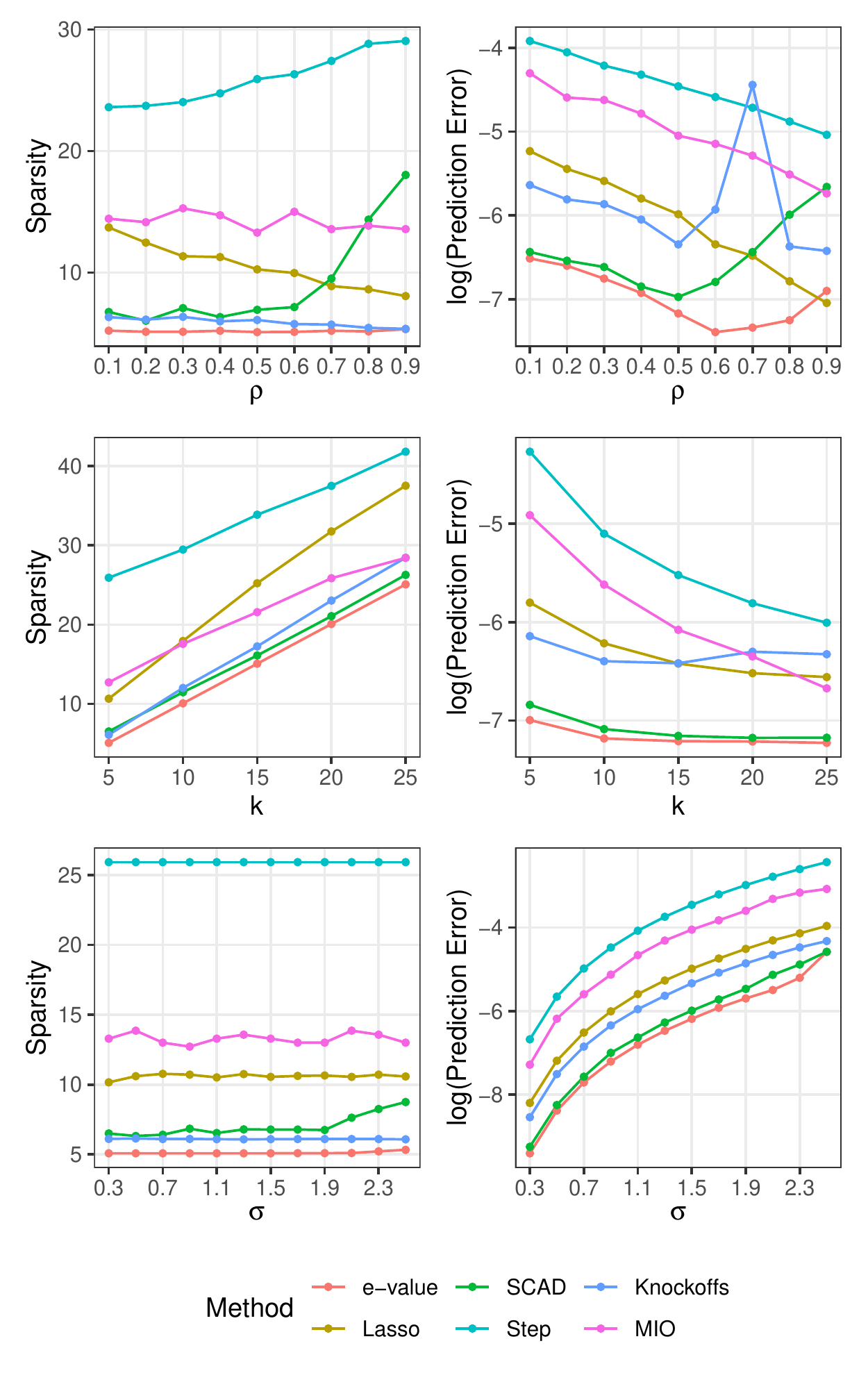}
\caption{
Performance metrics for $n=500, p=100$: top, middle and bottom rows indicate simulation settings 1, 2 and 3, respectively. Competing methods include LASSO and SCAD-penalized regression, Stepwise regression using BIC and backward deletion, Knockoff filters \cite{knockoff15}, and Mixed Integer Optimization \citep[MIO]{BertsimasEtal16}.
}
\label{fig:simplotsgamma}
\end{figure}


To implement e-values, we repeat Algorithm~\ref{alg:algoselectboot} for $ \tau_n \in \{ 0.2, 0.6, 1, 1.4, 1.8 \} \log (n) \sqrt{\log(p)} $, and select the model having the lowest GBIC$(\tau_n)$.


Figure~\ref{fig:simplotsgamma} summarizes the comparison results. Across all three settings, our method consistently produces the most predictive model, i.e. the model with the lowest prediction error. It also produces the sparsest model almost always. 
Among the competitors, SCAD performs the closest to e-values in setting 2 for both the metrics. However, SCAD seems to be more severely affected by high noise level (i.e. high $\sigma$) or high feature correlation (i.e. high $\rho$). Lasso and Step tend to select models with many null variables, and have higher prediction errors. Performance of the other two methods (MIO, knockoffs) is middling. 

\begin{table}[h!]
\centering
\scalebox{.85}{
    \begin{tabular}{l|llllll}
    \hline
    Method     & e-value & Lasso & SCAD & Step & Knockoff & MIO   \\\hline
    Time & 6.3    & 0.4  & 0.9 & 20.1 & 1.9       & 127.2 \\\hline
    \end{tabular}
    }
    \caption{Mean computation time (in seconds) for all methods.}
    \label{table:comptable}
\end{table}

We present computation times for Setting 1 averaged across all values of $\rho$ in Table~\ref{table:comptable}. All computations were performed on a Windows desktop with an 8-core Intel Core-i7 6700K 4GHz CPU and 16GB RAM. For e-values, using a smaller number of bootstrap samples or running computations over bootstrap samples in parallel greatly reduces computation time with little to no loss of performance. However, we report its computation time over a single core similar to other methods. Sparse methods like Lasso and SCAD are much faster as expected. Our method has lower computation time than Step and MIO, and much better performance. 



\subsection{High-dimensional linear regression ($p > n$)}
\label{subsec:sim2}

\begin{figure}[t!]
\centering
\includegraphics[width=1\columnwidth]{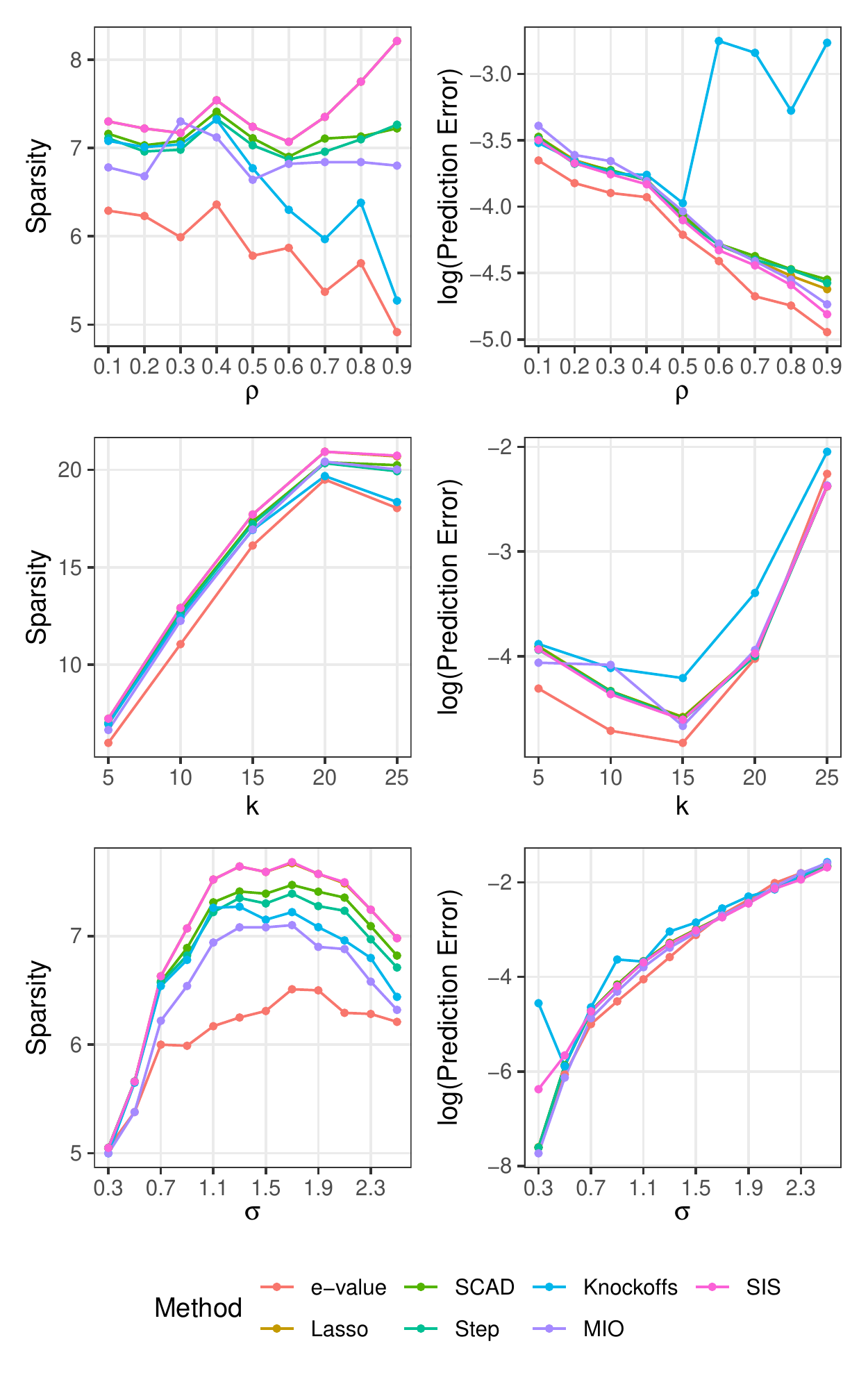}
\caption{Performance metrics for $n=100, p=500$.}
\label{fig:simplotsgamma2}
\end{figure}
We generate data from the same setup as Section~\ref{subsec:sim1}, but with $n = 100, p = 500$. We perform an initial screening of features using SIS, then apply the e-values and other methods on this SIS-selected predictor set. 
Figure~\ref{fig:simplotsgamma2} summarizes the results. In addition to competing methods, we report metrics corresponding to the original SIS-selected predictor set as a baseline. Sparsity-wise, e-values produce the most improvement over the SIS baseline among all methods, and tracks the true sparsity level closely. Both e-values and Knockoffs produce sparser estimates as $\rho$ grows higher (Setting 1). However, unlike the Knockoffs our method maintains good prediction performance even at high feature correlation. In general e-values maintain good prediction performance, although this difference is less obvious than the low-dimensional case. Note that for $k=25$ in setting 2, SIS screening produces overly sparse feature sets, affecting prediction errors for all methods.


\subsection{Mixed effect model}
\label{subsec:sim3}

%


\begin{table*}[t]
\centering
\scalebox{1}{
\begin{tabular}{llllllllll}
\hline
\multicolumn{2}{c}{Method} & 
\multicolumn{4}{c}{Setting 1: $n_i=5,m=30$} & \multicolumn{4}{c}{Setting 2: $n_i=10,m=60$}\\\cline{3-10}
& & FPR & FNR & Acc & MS & FPR & FNR & Acc & MS \\\hline
               & $\delta=0$                                   & 9.4  & 0.0  & 76 & 2.33 & 0.0  & 0.0  & 100& 2.00 \\
    ~          & $\delta=0.01$                                & 6.7  & 0.0  & 82 & 2.22 & 0.0  & 0.0  & 100& 2.00 \\
    e-value    & $\delta=0.05$                                & 1.0  & 0.0  & 97 & 2.03 & 0.0  & 0.0  & 100& 2.00 \\
    ~          & $\delta=0.1$									& 0.3  & 0.0  & 99 & 2.01 & 0.0  & 0.0  & 100& 2.00 \\
    ~          & $\delta=0.15$									& 0.0  & 0.0  & 100& 2.00 & 0.0  & 0.0  & 100& 2.00 \\\hline
									& BIC					& 21.5 & 9.9  & 49 & 2.26 & 1.5 & 1.9 & 86  & 2.10 \\
									& AIC                  & 17   & 11.0 & 46 & 2.43 & 1.5 & 3.3 & 77  & 2.20 \\
SCAD \cite{ref:PengLu_JMVA12109}	& GCV                  & 20.5 & 6    & 49 & 2.30 & 1.5 & 3   & 79  & 2.18 \\
									& $\sqrt{\log n/n}$    & 21   & 15.6 & 33 & 2.67 & 1.5 & 4.1 & 72  & 2.26 \\\hline
M-ALASSO \cite{ref:Bondelletal_Biometrics101069} &		& -    & -    & 73 & -    & -   & -   & 83  & -    \\
SCAD-P \cite{ref:FanLi_AoS122043}                &		& -    & -    & 90 & -    & -   & -   & 100 & -    \\
rPQL \cite{HuiEtal17}                            &		& -    & -    & 98 & -    & -   & -   & 99  & -    \\\hline
\end{tabular}
}
\caption{Performance comparison for mixed effect models. We compare e-values with a number of sparse penalized methods: (a) \citet{ref:PengLu_JMVA12109} that uses SCAD penalty and different methods of selecting regularization tuning parameters, (b) The adaptive lasso-based method of \citet{ref:Bondelletal_Biometrics101069}, (c) The SCAD-P method \citet{ref:FanLi_AoS122043}, and (d) regularized Penalized Quasi-Likelihood \citet[rPQL]{HuiEtal17}. For comparison with \citet{ref:PengLu_JMVA12109}, we present mean false positive (FPR) and false negative (FNR) rates, Accuracy (Acc), and Model Size (MS), i.e. the number of non-zero fixed effects estimated. To compare with other methods we only use Acc, since they did not report the rest of the metrics.
}
\label{table:simtablemixed}
\end{table*}

We use the simulation setup from \citet{ref:PengLu_JMVA12109}:
%
$ \bfY = \bfX \beta + \bfR \vecU + \bfepsilon$, where
%
the data $Z \equiv (Y,X,R)$ consists of $m$ independent groups of observations with multiple ($n_i$) dependent observations in each group, $\bfR$ being the within-group random effects design matrix. We consider 9 fixed effects and 4 random effects, with the true fixed effect coefficient $\bfbeta_0 = (1_2,0_7)$ and random effect coefficient $U$ drawn from $\cN_4(0, \Delta)$. The random effect covariance matrix $\Delta$ has elements $(\Delta)_{1 1} = 9$, $(\Delta)_{2 1} = 4.8$, $(\Delta)_{2 2} = 4$, $(\Delta)_{3 1} = 0.6$, $(\Delta)_{3 2} = (\Delta)_{3 3} = 1$, $(\Delta)_{4 j} = 0; j = 1, \ldots, 4$, and the error variance of $\epsilon$ is set at $\sigma^2 = 1$. The goal here is to select covariates of the fixed effect. We use two scenarios for our study: Setting 1, where the number of groups ($m$) considered is $30$, and the number of observations in the $i^{\Th}$ group, $i = 1, \ldots, m$, is $n_i = 5$, and Setting 2, where $m = 60, n_i = 10$. Finally, we generate 100 independent datasets for each setting. To implement e-values by minimizing GBIC$(\tau_n)$, we consider $\tau_n \in \{1, 1.2, \ldots, 4.8, 5 \}$. To tackle small sample issues in Setting 1 (Section~\ref{subsec:Algo}), we repeat the model selection procedure using the shifted e-values $\hat e^\delta_{rn}(\cdot)$ for $\delta \in \{0, 0.01, 0.05, 0.1, 0.15 \}$.

Without shifted thresholds, e-values perform commendably in both settings. For Setting 2, it reaches the best possible performance across all metrics. 
However, we observed that in a few replications of setting 1, a small number of null input features had e-values only slightly lower than the full model e-value, resulting in increased FPR and model size. We experimented with lowered detection thresholds to mitigate this. As seen in Table~\ref{table:simtablemixed}, increasing $\delta$ gradually eliminates the null features, and e-values reach perfect performance across all metrics for $\delta=0.15$.
\section{Real data examples}
\paragraph{Indian monsoon data}

\begin{figure*}[t!]
\centering
\subfigure[]{
\label{fig:tstat}
\includegraphics[width=.5\linewidth]{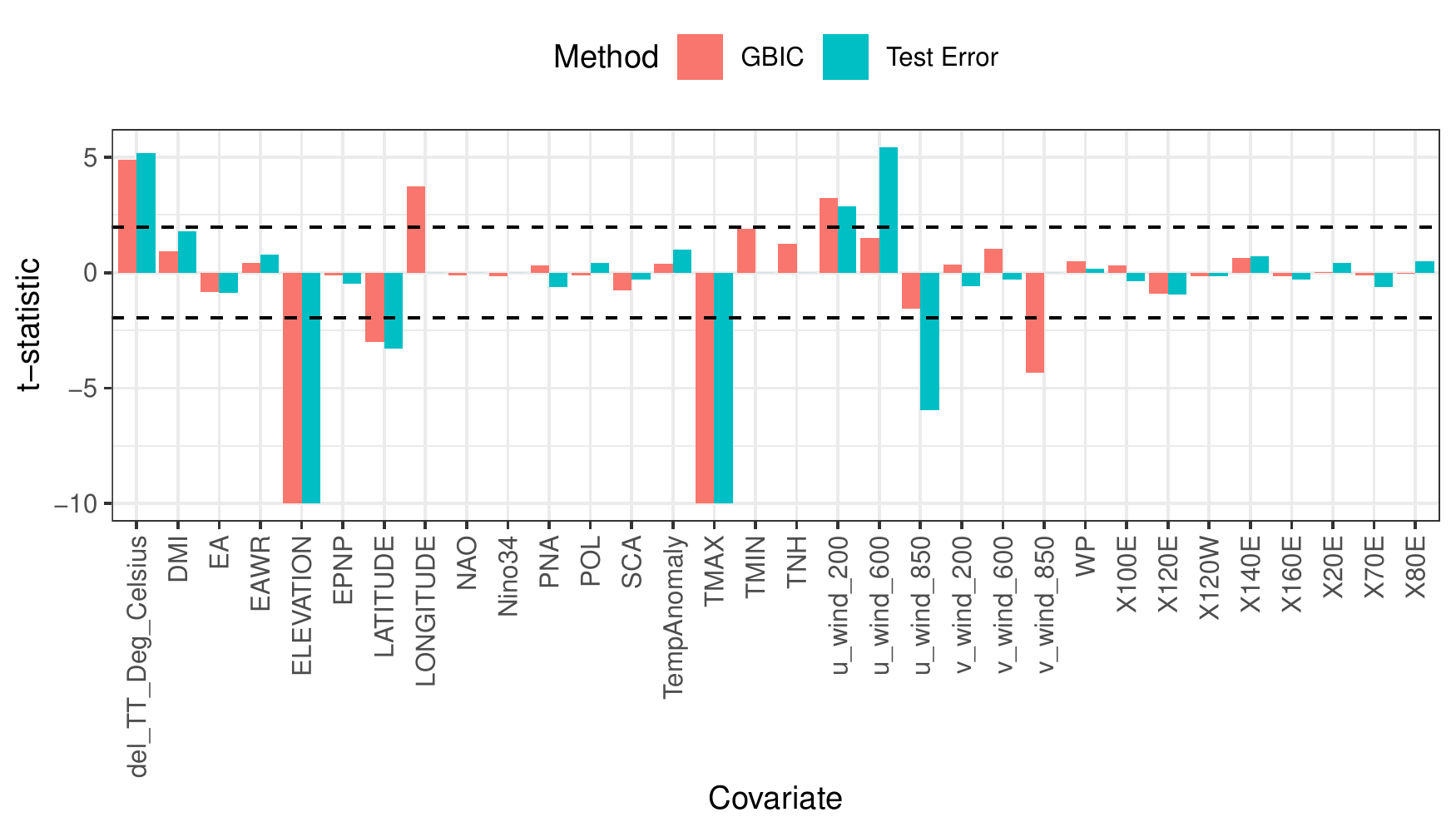}
}
\subfigure[]{
\label{fig:figa}
\includegraphics[width=.25\linewidth]{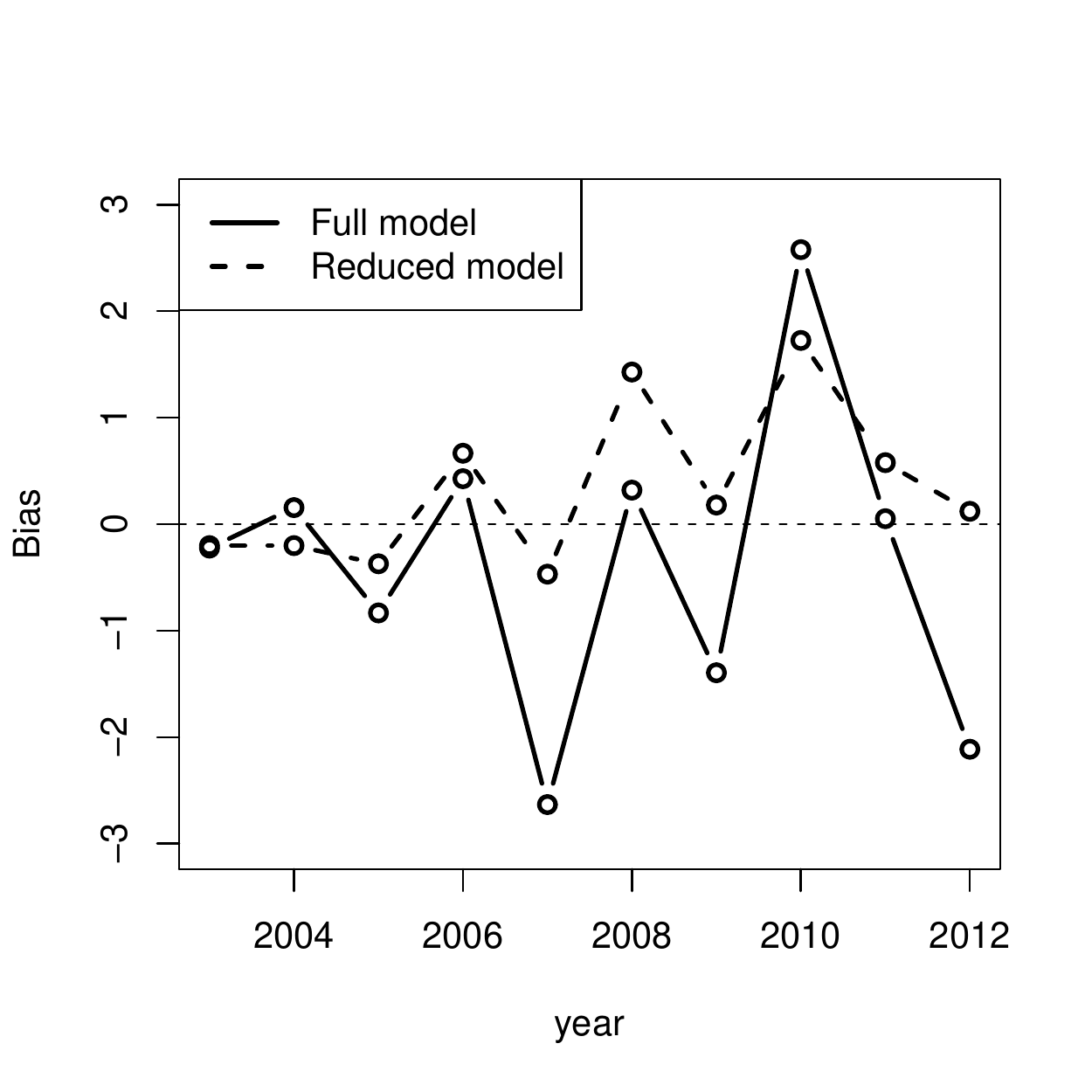}
}
\subfigure[]{
\label{fig:figd}
\includegraphics[width=.2\linewidth]{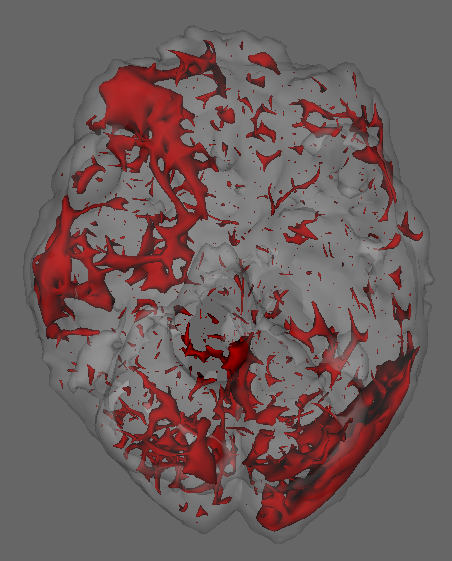}
}
\caption{(a) Plot of t-statitics for Indian Monsoon data for features selected by at least one of the methods. Dashed lines indicate standard Gaussian quantiles at tail probabilities 0.025 and 0.975, (b) Full vs. reduced model rolling forecast comparisons, and (c) fMRI data: smoothed surface obtained from p-values show high spatial dependence in right optic nerve, auditory nerves and auditory cortex (top left), left visual cortex (bottom right) and cerebellum (lower middle).}
\label{fig:realdata}
\end{figure*}

To identify the driving factors behind precipitation during the Indian monsoon season using e-values, we obtain data on 35 potential covariates (see Appendix~\ref{Section:IndianMonsoon}) from National Climatic Data Center (NCDC) 
and National Oceanic and Atmospheric Administration (NOAA) repositories for 1978--2012. We consider annual medians of covariates as fixed effects, log yearly rainfall at a weather station as output feature, and include year-specific random intercepts. To implement e-values, we use projection depth \citep{zuo03} and GBS resample sizes $R = R_{1} = 1000$. We train our model on data from the years 1978-2002, run e-values best subset selection for tuning parameters $\tau_n \in \{ 0.05, 0.1, \ldots, 1\}$. We consider two methods to select the best refitted model: (a) minimizing GBIC$(\tau_n)$, and (b) minimizing forecasting errors on samples from 2003--2012.


Figure~\ref{fig:realdata}a plots the t-statistics for features from the best refitted models obtained by the above two methods. Minimizing for GBIC and test error selects 32 and 26 covariates, respectively. The largest contributors are maximum temperature and elevation, which are related to precipitation based on the Clausius-Clapeyron relation \citep{Lietal17,SingletonToumi12}. All other selected covariates have documented effects on Indian monsoon \citep{KrishChapter,MoonWangHa12}. Reduced model forecasts obtained from a rolling validation scheme (i.e. i.e. use data from 1978--2002 for 2003, 1979-2003 for 2004 and so on) have less bias across testing years (Fig. \ref{fig:realdata}b).
\paragraph{Spatio-temporal dependence analysis in fMRI data}
\label{subsec:fmri}

This dataset is due to \citet{WakemanHenson15}, where each of the 19 subjects go through 9 runs of consecutive visual tasks. Blood oxygen level readings are recorded across time as 3D images made of $64 \times 64 \times 33$ (total 135,168) voxels. Here we use the data from a single run and task on subject 1, and aim to estimate dependence patterns of readings across 210 time points and areas of the brain.
We fit separate regressions at each voxel (Appendix~\ref{sec:fmri}), with second order autoregressive terms, neighboring voxel readings and one-hot encoded visual task categories in the design matrix. After applying the e-value feature selection, we compute the F-statistic at each voxel using selected coefficients only, and obtain their p-values. Fig.~\ref{fig:realdata}c highlights voxels with p-values $< 0.05$. Left and right visual cortex areas show high spatial dependence, with more dependence on the left side. Signals from the right visual field obtained by both eyes are processed by the left visual cortex. The lop-sided dependence pattern suggests that visual signals from the right side led to a higher degree of processing in our subject's brain. We also see activity in the cerebellum, the role of which in visual perception is well-known \citep{CalhounEtal10, Kirschen10}.

\section{Conclusion}
\label{Section:Conclusion}

In this work, we introduced a new paradigm of feature selection through {\it e-values}. The e-values can be particularly helpful in situations where model training is costly and potentially distributed across multiple servers (i.e. federated learning), so that a brute force parallelized approach of training and evaluating multiple models is not practical or even possible. 

There are three immediate extensions of the framework presented in this paper. Firstly, grouped e-values are of interest to leverage prior structural information on the predictor set. There are no conceptual difficulties in evaluating overlapping {\it and} non-overlapping groups of predictors using e-values in place of individual predictors. However technical conditions may be required to ensure a rigorous implementation. Secondly, our current formulation of e-values essentially relies upon the number of samples being more than the effective number of predictors for a unique covariance matrix of the parameter estimate to asymptotically exist. When $p>n$, using a sure screening method to filter out unnecessary variables works well empirically. However this needs theoretical validation. 
Thirdly, instead of using mean depth, other functionals of the (empirical) depth distribution---such as quantiles---can be used as e-values. Similar to stability selection \cite{stabsel}, it may be possible to use the intersection of predictor sets obtained by using a number of such functionals in Algorithm~\ref{alg:algoselectboot} as the final selected set of important predictors.

An effective implementation of e-values hinges on the choice of bootstrap method and tuning parameter $\tau_n$. To this end, we see opportunity for enriching the research on empirical process methods in complex overparametrized models, such as Deep Neural Nets (DNN), which the e-values framework can build up on. Given the current push for interpretability and trustworthiness of DNN-based decision making systems, there is potential for tools to be developed within the general framework of e-values that provide local and global explanations of large-scale deployed model outputs in an efficient manner.

\section*{Acknowledgements}
This work is part of the first author (SM)'s PhD thesis \citep{mythesis}. He acknowledges the support of the University of Minnesota Interdisciplinary Doctoral Fellowship during his PhD. The research of SC is partially supported by the US National Science Foundation grants 1737918, 1939916 and 1939956 and a grant from Cisco Systems.

\section*{Reproducibility}
Code and data for the experiments in this paper are available at \url{https://github.com/shubhobm/e-values}.

\bibliographystyle{icml2022}
\bibliography{main}

\newpage
\appendix
\onecolumn
\section*{Appendix}
\section{Consistency of Generalized Bootstrap}
\label{app:boot}
We first state a number of conditions on the energy functions $\psi_i(\cdot,\cdot)$, under which we state and prove two results to ensure consistency of the estimation and bootstrap approximation procedures. In the context of the main paper, the conditions and results here ensure that the full model parameter estimate $\hat \theta_{*}$ follows a Gaussian sampling distribution, and Generalized Bootstrap (GBS) can be used to approximate this sampling distribution.

\subsection{Technical conditions}
Note that several sets of alternative conditions can be developed \cite{ref:CBose_AoS05414,Lahiri92}, many of which are 
amenable to our results. However, for the sake of brevity and clarity, 
we only address the case where  the energy function
$\psi_{i}  ( \cdot, \cdot )$ is smooth in the first argument. This case covers a vast number of models  routinely considered in statistics and machine learning. 

We often drop the second argument from energy function, thus for example 
$ \psi_{i}  \bigl( \theta \bigr) \equiv  \psi_{i}  \bigl( \theta, Z_i \bigr)$, 
and use the notation $\hat{\psi}_{ k i}$ for $\psi_{k i} (\hat{\vectheta}_{*})$, 
for $k = 0, 1, 2$. Also, for any function $h (\vectheta)$ evaluated at the true parameter value $\vectheta_{*}$, we use the notation $h \equiv h (\vectheta_{*})$. When $A$ and $B$ are 
square matrices of identical dimensions, the notation $B < A$ implies that 
the matrix $A - B$ is positive definite. 

In a neighborhood of $\vectheta_{*}$, we assume the functions $\psi_{i}$ are 
thrice continuously differentiable in the first argument, with the successive  
derivatives denoted by $\psi_{k i}$, $k = 0, 1, 2$. That is, there exists a $\delta > 0$ such that for any $\theta = \vectheta_{*} + t$ satisfying 
$\| t \| < \delta$ we have
\ban 
{\frac {d}{ d \theta}} \psi_{i} (\theta)  & := \psi_{0 i} (\theta) \in \BR^{p}, 
\ean 
and for the $a$-th element  of $\psi_{0 i} (\theta)$, denoted by 
$\psi_{0 i (a)} (\theta)$,  we have  
\ban
\psi_{0 i (a)} (\theta) & = \psi_{0 i (a)} (\vectheta_{*}) 
	+ \psi_{1 i (a)} (\vectheta_{*}) t  
	+ 2^{-1}  t^{T} \psi_{2 i (a)} (\vectheta_{*} + c t) t,
\ean
for $a = 1, \ldots p$, and some $ c \in (0, 1)$ possibly depending on $a$. 
We assume that for each  $n$, there is a sequence of $\sigma$-fields 
$\cF_{1} \subset \cF_{2} \ldots \subset \cF_{n}$ such that 
$\{ \sum_{i = 1}^{j}  \psi_{0 i } (\vectheta_{*}), \cF_{j} \}$ is a martingale. 

The spectral decomposition of $\Gamma_{0 n} :=  \sum_{i = 1}^{n} \BE \psi_{0 i} \psi_{0 i}^{T}$  is given by 
$\Gamma_{0 n}  = \matP_{0 n} \vecLambda_{0 n} \matP^{T}_{0 n}$, 
where $\matP_{0 n} \in \BR^{ p} \times \BR^{p}$ is an orthogonal matrix
whose columns contain the eigenvectors, 
and $\vecLambda_{0 n}$ is a diagonal matrix containing the eigenvalues of 
$\Gamma_{0 n}$. We assume that $\Gamma_{0 n}$ is positive 
definite, that is, all the diagonal entries of $\vecLambda_{0 n}$ are positive numbers. We assume that there is a constant $\delta_{0 } > 0$ such that 
$\lambda_{min} (\Gamma_{0 n}) > \delta_{ 0 }$ for sufficiently large $n$.

Let $\Gamma_{1 i}(\vectheta_{*})$ be the $p \times p$ matrix whose $a$-th row is $\BE \psi_{1 i (a)}$; we assume this expectation exists. Define $\Gamma_{1 n} = \sum_{ i =1}^{n} \Gamma_{1 i} (\vectheta_{*})$. We assume that $\Gamma_{1 n}$ is nonsingular for each  $n$. The singular value decomposition of $\Gamma_{1 n}$ is given by 
$\Gamma_{1 n}  =  \matP_{1 n} \vecLambda_{1 n} \matQ^{T}_{1 n}$, where $\matP_{1 n},  \matQ_{1 n}  \in \BR^{ p} \times \BR^{p}$ 
are orthogonal matrices, and $\vecLambda_{1 n}$ is a diagonal  matrix. We assume that the diagonal entries of $\vecLambda_{1 n}$ are all positive, which implies that \textit{ in the population, at the true value of the parameter} the energy functional $\Psi_n(\theta_*)$ actually achieves a minimal value. We define $\vecLambda_{k n}^{c}$ for various real numbers $c$ as diagonal matrices where the $j$-th diagonal entry of $\vecLambda_{k n}$ is raised to the power $c$, for $k = 0, 1$. Correspondingly, we 
define $\Gamma_{1 n}^{c} =   \matP_{1 n} \vecLambda_{1 n}^{c} \matQ^{T}_{1 n}$. We assume that there is a constant $\delta_{1 } > 0$ such that $\lambda_{max} (\Gamma_{1 n}^{T} \Gamma_{1 n}) < \delta_{1 }$ for all sufficiently large $n$.
Define the matrix 
$A_{n} = 
 \Gamma_{0 n}^{-1/2} \Gamma_{1 n}$. 
We assume the following conditions: 

\renewcommand{\theequation}{C.\arabic{equation}}
\begin{enumerate}[nolistsep]
\item[(C1)] 
The minimum eigenvalue of $A^{T}_{n} A_{n}$ tends to infinity. That is, there is a 
sequence $a_{n} \uparrow \infty$ as $n \raro \infty$ such that 
\baq 
\lambda_{min} \bigl( \Gamma_{1 n} \Gamma_{0 n}^{-1}  \Gamma_{1 n}^{T} \bigr)
\asymp a_{n}^{2}. 
\label{eq:MinEigenvalueRate}
\eaq

\item[(C2)] 
There exists a sequence of positive reals
$\{ \gamma_{n}\}$ that is bounded away from zero, such that
\baq
\lambda_{max} \bigl( \Gamma_{1 n}^{-1} \Gamma_{0 n}^{2} \Gamma_{1 n}^{-T} \bigr)
= o (\gamma_{n}^{-2}) \ \ \text{ as $n \raro \infty$.} 
\label{eq:Trace} 
\eaq 

\item[(C3)] 
\baq 
\BE  \left\| A_{n}^{-1} 
 \bigl( \sum_{i = 1}^n \psi_{1 i} - \Gamma_{1 n } \bigr) 
 A_{n}^{-1} \right\|_{F}^{2} = o (p \gamma_{n}^{-2}).
 \label{eq:FrobSq}
 \eaq
 where $\|A\|_{F}$ denotes the Frobenius norm of matrix $A$.

\item[(C4)] 
For the symmetric matrix $\psi_{2 i (a)} (\theta)$ and for some $\delta_2 > 0$, there exists a symmetric matrix $M_{2  i (a)}$ such that
\ban 
\sup_{ \| \theta - \vectheta_{*} \| < \delta_2} 
 \psi_{2 i (a)} (\theta) < M_{2 i (a)}, 
 \ean
satisfying 
\baq
\sum_{a = 1}^p \sum_{i = 1}^n 
 \BE \lambda_{max}^{2} \bigl( M_{2 i (a)} \bigr)
  & = o \bigl( a_{n}^{6}  n^{-1} p  \gamma_{n}^{-2} \bigr).
  \label{eq:2ndMomentBound}
\eaq

\item[(C5)] 
For any vector $c \in \BR^{ p}$ with $\|c\| = 1$, we define the random variable $Z_{n i} = - c^{T} \Gamma_{0 n}^{-1/2} \psi_{i}$ for $i =1, \ldots n$. We assume that 
\baq 
\sum_{i =1}^{n} Z_{n i}^{2} \praro 1, 
\text{ and } 
\BE \bigl[ \max_{i} \| Z_{n i} \| \bigr] \raro 0. 
\label{eq:CLTConditions}
\eaq

\item[(C6)] 
 Assume that
\baq 
\lambda_{max} \bigl( \Gamma_{1 n} \Gamma_{0 n}^{-1}  \Gamma_{1 n}^{T} \bigr)
\asymp a_{n}^{2}. 
\label{eq:MaxEigenvalueRate}
\eaq
\end{enumerate}
\renewcommand{\theequation}{\thesection.\arabic{equation}}

The technical conditions (C1)-(C5) are extremely broad, and allow for different rates of convergence of different parameter estimators. The additional condition (C6) is a natural condition that, coupled with (C1), ensures identical rate of convergence $a_{n}$ for all the parameter estimators in a model.

Standard regularity conditions on likelihood and estimating functions that have been routinely assumed in the literature are special cases of the framework above. In such cases,  (C1)-(C6)  hold with $a_{n} \equiv {n}^{1/2}$, 
resulting in the standard ``root-$n$'' asymptotics.

\subsection{Results}
We first present the consistency and asymptotic normality of the estimation process in Theorem~\ref{Theorem:CLT} below.

\begin{theorem}
\label{Theorem:CLT}
Assume conditions (C1)-(C5). Then $\hat{\vectheta}_{*}$ is a consistent estimator of  ${\vectheta}_{*}$, and
$A_{n} (\hat{\vectheta}_{*} - {\vectheta}_{*})$ converges weakly to the $p$-dimensional standard Gaussian distribution. Under the additional condition (C6), we have that $a_{n} (\hat{\vectheta}_{*} - {\vectheta}_{*})$ converges weakly to a Gaussian distribution in $p$-dimension.
\end{theorem}

\begin{proof}[Proof of Theorem~\ref{Theorem:CLT}]

We consider a generic point $\vectheta = \vectheta_{ *} + A_{ n}^{-1} t$. 
From the Taylor series expansion, we have
\ban
\psi_{0i (a)} (\vectheta) & = \psi_{0i (a)}
	+ \psi_{1 i (a)} A_{ n}^{-1} t 
	+ 2^{-1}  t^{T} A_{n}^{-T} \psi_{2 i (a)} (\tilde{\vectheta}_{*} ) 
A_{ n}^{-1} t,
\ean
for $ a = 1, \ldots p$, and $\tilde{\vectheta}_{*} = {\vectheta}_{*}  + c A_{n}^{-1} t$ for some $c \in (0, 1)$.

Recall our convention that for any function $h (\vectheta)$ 
evaluated at the true parameter value $\vectheta_{*}$, we use the notation 
$h \equiv h (\vectheta_{*})$. Also define the $p$-dimensional vector 
$R_n (\tilde{\vectheta}_{*},  t)$ whose $a$-th element is given by 
\ban 
R_{n(a)} (\tilde{\vectheta}_{*},  t)
=  t^{T} A_{ n}^{-T} 
\sum_{i =1}^{n} \psi_{2 i (a)} (\tilde{\vectheta}_{*} ) A_{ n}^{-1} t.
\ean

Thus we have 
\ban 
p^{-1/2} A_{n}^{-1} \sum_{i = 1}^n \psi_{0 i} 
(\vectheta_{ *} + A_{n}^{-1} t)
& = p^{-1/2} A_{n}^{-1} \sum_{i = 1}^{ n} \psi_{0 i}  
	+  p^{-1/2} A_{n}^{-1} \sum_{i = 1}^{ n} \psi_{1 i} A_{n}^{-1} t 
+ 2^{-1}  p^{-1/2} A_{n}^{-1} R_n (\tilde{\vectheta}_{* },  t) \\
& = p^{-1/2} A_{n}^{-1} \sum_{i = 1}^{ n} \psi_{0 i}  
	+ p^{-1/2} A_{n}^{-1} \Gamma_{1 n } A_{n}^{-1} t \\
&\quad + p^{-1/2} A_{n}^{-1} 
\bigl( \sum_{i = 1}^{ n} \psi_{1 i} - \Gamma_{1 n } \bigr) 
A_{n}^{-1} t \\
&\quad + 2^{-1} p^{-1/2}  A_{n}^{-1} R_n (\tilde{\vectheta}_{*},  t).
\ean

Fix $\epsilon > 0$. We first show that there exists a $C_{0} > 0$ such that
\baq 
\BP \Bigl[ 
\bigl\| p^{-1/2}  A_{n}^{-1}  \sum_{i = 1}^{ n} \psi_{0 i} \bigr\| 
> C_{0} \Bigr] < \epsilon/2. 
\label{eq:Bound1}
\eaq
For this, we compute 
\ban
p^{-1} \BE \bigl\| A_{n}^{-1}  \sum_{i = 1}^{ n} \psi_{0 i} \bigr\|^{2}
& = p^{-1} \BE  \sum_{i, j = 1}^{ n} \psi_{0 i}^{T} 
A_{n}^{-T} A_{n}^{-1}   \psi_{0 j} \\
& = p^{-1} \trace \left( A_{n}^{-T} A_{n}^{-1} 
\BE  \sum_{i = 1}^{ n} \psi_{0 i} \psi_{0 i}^{T} \right)\\
& = p^{-1} \trace \left( A_{n}^{-T} A_{n}^{-1}  \Gamma_{0 n} \right)\\
& = O (1)
\ean 
from assumption \eqref{eq:Trace}. 
 
Define 
\ban 
S_n (t) = 
p^{-1/2} A_{n}^{-1} \bigl( \sum_{i = 1}^{ n} \psi_{0 i} 
(\vectheta_{ *} + A_{n}^{-1} t) 
- \sum_{i = 1}^{ n} \psi_{0 i}  \bigr) 
- p^{-1/2}  \Gamma_{1 n }^{-1} \Gamma_{0 n } t. 
\ean

We next show that for any $C > 0$, for all sufficiently large $n$, we have 
\baq 
\BE \Bigl[ \sup_{ \| t\| \leq C} \bigl\| S_{n} (t) \bigr\| \Bigr]^{2} = o (1). 
\label{eq:Bound2}
\eaq
This follows from \eqref{eq:FrobSq} and \eqref{eq:2ndMomentBound}. Note that 
 \ban 
 S_{n} (t) = 
p^{-1/2}  A_{n}^{-1} 
 \bigl( \sum_{i = 1}^{ n} \psi_{1 i} - \Gamma_{1 n } \bigr) 
 A_{n}^{-1} t
+ 2^{-1} p^{-1/2}  A_{n}^{-1} R_{ n} (\tilde{\vectheta}_{n },  t).
\ean
Thus, 
\ban 
\sup_{ \| t\| \leq C} \bigl\| S_{n} (t) \bigr\| 
& \leq 
p^{-1/2} \sup_{ \| t\| \leq C} \bigl\| 
 A_{n}^{-1} 
 \bigl( \sum_{i = 1}^{ n} \psi_{1 i} - \Gamma_{1 n } \bigr) 
 A_{n}^{-1} t \bigr\| 
+ 2^{-1} p^{-1/2} \sup_{ \| t\| \leq C} \bigl\| 
 A_{n}^{-1} R_{ n} (\tilde{\vectheta}_{*},  t) \bigr\|. 
\ean
We consider each of these terms separately. 

For any matrix $M \in \BR^{p} \times \BR^{p}$, we have
\ban 
\sup_{ \| t\| \leq C} \bigl\| M t \bigr\| 
= \sup_{ \| t\| \leq C} 
\Bigl[ \sum_{i =1}^{p} \bigl( \sum_{j =1}^{p} M_{ i j} t_{j} \bigr)^{2} \Bigr]^{1/2}
\leq \sup_{ \| t\| \leq C} 
\Bigl[ \sum_{i =1}^{p} \sum_{j =1}^{p} M_{ i j}^{2} 
\sum_{j =1}^{p} t_{j}^{2} \Bigr]^{1/2} 
= \bigl\| M \bigr\|_{F} \sup_{ \| t\| \leq C} \|t\| 
= C \bigl\| M \bigr\|_{F}.
\ean

Using $M =  A_{n}^{-1} \bigl( \sum_{i = 1}^{ n} \psi_{1 i} - \Gamma_{1 n } \bigr) A_{n}^{-1}$ and \eqref{eq:FrobSq}, we get one part of the result. 

For the other term, we similarly have 
\ban 
\Bigl[ \sup_{ \| t\| \leq C} \bigl\| 
p^{-1/2}  A_{n}^{-1} R_{ n} (\tilde{\vectheta}_{*},  t)
\bigr\| \Bigr]^{2}
& = p^{-1} \sup_{ \| t\| \leq C} 
\bigl\|  A_{n}^{-1} R_{n} (\tilde{\vectheta}_{*},  t)
\bigr\|^{2} \\
& \leq p^{-1} \lambda_{max}\bigl( A_{n}^{-T} A_{n}^{-1}  \bigr)
\sup_{ \| t\| \leq C} \bigl\| R_{n} (\tilde{\vectheta}_{*},  t)\bigr\|^{2}\\
& \leq p^{-1} \lambda_{max}\bigl( A_{n}^{-1} A_{n}^{-T}  \bigr)
\sup_{ \| t\| \leq C} \bigl\| R_{n} (\tilde{\vectheta}_{*},  t)\bigr\|^{2}\\
& \leq p^{-1} a_{n}^{-2} 
\sup_{ \| t\| \leq C} \bigl\| R_{n} (\tilde{\vectheta}_{*},  t)\bigr\|^{2}.
\ean

Note that 
\ban 
\bigl( \sup_{ \| t\| \leq C} \bigl\| R_n (\tilde\theta_*,  t)\bigr\| \bigr)^{2}
= 
\sup_{ \| t\| \leq C} \bigl\| R_n (\tilde\theta_*,  t)\bigr\|^{2}.
\ean
Now 
\ban 
\bigl\| R_n (\tilde\theta_*,  t)\bigr\|^{2}
& = \sum_{a = 1}^{ p} 
\bigl( R_{ * n (a)} (\tilde\theta_*,  t) \bigr)^{2} \\
& = \sum_{a = 1}^{ p} 
\bigl( 
t^{T} A_{n}^{-T} 
\sum_{i =1}^{n} \psi_{2 i (a)} (\tilde\theta_* ) A_{n}^{-1} t
\bigr)^{2} \\
& = \sum_{a = 1}^{ p} 
\sum_{i, j = 1}^{ n} 
t^{T} A_{n}^{-T} \psi_{2 i (a)} (\tilde\theta_* ) A_{n}^{-1} t
t^{T} A_{n}^{-T} \psi_{2 j(a)} (\tilde\theta_* ) A_{n}^{-1} t.
\ean
Based on this, we have 
\ban 
\sup_{ \| t\| \leq C} \bigl\| R_n (\tilde\theta_*,  t)\bigr\|^{2}
& = \sup_{ \| t\| \leq C} 
\sum_{a = 1}^{ p} \sum_{i, j = 1}^{ n} 
t^{T} A_{n}^{-T} \psi_{2 i (a)} (\tilde\theta_* ) A_{n}^{-1} t
t^{T} A_{n}^{-T} \psi_{2 * n j (a)} (\tilde\theta_* ) A_{n}^{-1} t
\\
& \leq \sup_{ \| t\| \leq C} 
\sum_{a = 1}^{ p} 
\sum_{i, j = 1}^{ n} 
t^{T} A_{n}^{-T} M_{2  i (a)}  A_{n}^{-1} t
t^{T} A_{n}^{-T} M_{2  j (a)}  A_{n}^{-1} t
\\
& \leq \sup_{ \| t\| \leq C} \bigl\| A_{n}^{-1} t \bigr\|^{4}
\sum_{a = 1}^{ p} 
\Bigl(
\sum_{i = 1}^{ n} 
 \lambda_{max} \bigl( M_{2 i (a)} \bigr) \Bigr)^{2}
 \\
& \leq  C^{4} n \lambda_{max}^{2} \bigl( A_{n}^{-T} A_{n}^{-1} \bigr)
\sum_{a = 1}^{ p} 
\sum_{i = 1}^{ n} 
 \lambda_{max}^{2} \bigl( M_{2 i (a)} \bigr).
\ean
 
Putting all these together, we have 
\ban 
\BE \Bigl[ \sup_{ \| t\| \leq C} \bigl\|
 p^{-1/2}  A_{n}^{-1} R_n (\tilde\theta_*,  t) \bigr\|
 \Bigr]^{2}
& = p^{-1}  \BE \Bigl[ \sup_{ \| t\| \leq C} 
  A_{n}^{-1} R_n (\tilde\theta_*,  t)
 \Bigr]^{2} \\
& \leq p^{-1}  a_{n}^{-2}   \BE \Bigl[ \sup_{ \| t\| \leq C} 
  \bigl\| R_n (\tilde\theta_*,  t)\bigr\|
 \Bigr]^{2} \\
 & = O \bigl( p^{-1} a_{n}^{-2} \bigr)
 \BE \Bigl[ \sup_{ \| t\| \leq C} 
  \bigl\| R_n (\tilde\theta_*,  t)\bigr\|
 \Bigr]^{2} \\
  & = O \bigl( p^{-1} n a_{n}^{-6} \bigr)
\sum_{a = 1}^{ p} 
\sum_{i = 1}^{ n} 
 \BE \lambda_{max}^{2} \bigl( M_{2 n i (a)} \bigr)\\
& = o(1),
\ean
using \eqref{eq:2ndMomentBound}.

Since we have defined 
\ban 
S_{n} (t) = 
p^{-1/2} A_{n}^{-1} \bigl( \sum_{i = 1}^{ n} \psi_{0 i} 
(\vectheta_{ *} + A_{n}^{-1} t) 
- \sum_{i = 1}^{ n} \psi_{0 i}  \bigr) 
- p^{-1/2}  \Gamma_{1 n }^{-1} \Gamma_{0 n } t, 
\ean
we have
\ban 
p^{-1/2} A_{n}^{-1}  \sum_{i = 1}^{ n} \psi_{0 i} 
(\vectheta_{ *} +  p^{1/2} A_{n}^{-1} t) 
&= S_{n} (t) + 
p^{-1/2} A_{n}^{-1}   \sum_{i = 1}^{ n} \psi_{0 i} 
+ A_{n}^{-1} \Gamma_{1 n } A_{n}^{-1} t.
\ean

Hence
\ban 
&\inf_{ \| t\| = C} 
\Bigl\{ p^{-1/2} 
t^{T} \Gamma_{1 n } A_{n}^{-1}  \sum_{i = 1}^{ n} \psi_{0 i} 
(\vectheta_{ *} +  p^{1/2} A_{n}^{-1} t) 
\Bigr\} \\
& = \inf_{ \| t\| = C} 
\Bigl\{ 
t^{T}  \Gamma_{1 n } S_{n} (t) + 
p^{-1/2} t^{T}  \Gamma_{1 n }  A_{n}^{-1}  \sum_{i = 1}^{ n} \psi_{0 i}  
+ t^{T}  \Gamma_{1 n } A_{n}^{-1} \Gamma_{1 n } A_{n}^{-1} t
\Bigr\} \\
& \geq 
 \inf_{ \| t\| = C} t^{T}  \Gamma_{1 n } S_{n} (t) 
 + p^{-1/2}  \inf_{ \| t\| = C} t^{T}  \Gamma_{1 n }  A_{n}^{-1}  
 \sum_{i = 1}^{ n} \psi_{0 i}  
+ \inf_{ \| t\| = C} t^{T}  \Gamma_{1 n } A_{n}^{-1} \Gamma_{1 n } A_{n}^{-1} t
\\
 & \geq  - C  \delta_{1}^{1/2}  \sup_{ | t| = C} | S_{n} (t) |
- C \delta_{1}^{1/2}  p^{-1/2}  
 \| A_{n}^{-1}  \sum_{i = 1}^{ n} \psi_{0 i} \| + C^{2} \delta_{0}.
\ean
The last step above utilizes facts like $a^{T} b \geq - \|a\| \|b\|$. Consequently, defining $C_{1} = C \delta_{ 0 }/\delta_{1 }^{1/2}$, we have 
\ban 
& \BP \Bigl[  
\inf_{ \| t\| = C} 
\Bigl\{ 
t^{T} \Gamma_{1 n }  A_{n}^{-1}  \sum_{i = 1}^{ n} \psi_{0i} (\vectheta_{ *} +  p^{1/2} A_{n}^{-1} t) 
\Bigr\} < 0 
\Bigr] \\
& \leq \BP \Bigl[  
\sup_{ \| t\| = C} | S_{n} (t) | + | A_{n}^{-1}  \sum_{i = 1}^{ n} \psi_{0i} | 
> C_{1} \Bigr] \\
& \leq \BP \Bigl[  
\sup_{ \| t\| = C} \bigl\| S_{n} (t) \bigr\|  > C_{1}/2 \Bigr] 
+ \BP \Bigl[ \bigl\| A_{n}^{-1}  \sum_{i = 1}^{ n} \psi_{0 i} \bigr\| 
> C_{1}/2 
\Bigr] < \epsilon, 
\ean
for all sufficiently large $n$,  using \eqref{eq:Bound1} and \eqref{eq:Bound2}. 
This implies that with a probability greater than $1 - \epsilon$ there is a  root $T_{n}$ of the equations $\sum_{i = 1}^{ n} \psi_{0i} 
(\vectheta_{ *} +  A_{n}^{-1} t)$  in the ball $\{ \| t\| < C \}$, for some $C > 0$ and all sufficiently large $n$. Defining $\hat{\vectheta}_{ *}  = 
\vectheta_{ *} +   A_{n}^{-1} T_{n}$, we obtain the desired result. 
Issues like dependence on $\epsilon$ and other technical details 
are handled using standard arguments, see \citet{ref:CBose_AoS05414} for related arguments.

 Since we have $\sup_{\| t\| < C} \| S_{n} (t) \| = o_{P} (1) $, and $T_{n}$ lies in the set $\| t\| < C$, define  $- R_{n} = S_{n} T_{n} = o_{P} (1)$. Consequently
\ban 
- R_{n} & =  S_{n} T_{n}  \\
& = 
 p^{-1/2} A_{n}^{-1} \bigl( \sum_{i = 1}^{ n} \psi_{0i} 
(\vectheta_{ *} + A_{n}^{-1} T_{n}) 
- \sum_{i = 1}^{ n} \psi_{0 i}  \bigr) 
- p^{-1/2}  \Gamma_{1 n }^{-1} \Gamma_{0 n }T_{n}
\\
& = p^{-1/2} A_{n}^{-1} \sum_{i = 1}^{ n} \psi_{0 i} 
- p^{-1/2}  \Gamma_{1 n }^{-1} \Gamma_{0n } T_{n}. 
\ean
Thus, 
\ban 
T_{n} = 
-  \Gamma_{0 n }^{-1} \Gamma_{1 n } A_{n}^{-1} 
\sum_{i = 1}^{ n} \psi_{0 i} + p^{1/2} 
\Gamma_{0 n }^{-1} \Gamma_{1 n } R_{n}
= 
- \Gamma_{0 n }^{-1/2}  \sum_{i = 1}^{ n} \Psi_{0 i} 
+ p^{1/2} \Gamma_{0 n }^{-1} \Gamma_{1 n } R_{n}.
\ean

Note that our conditions imply that for any $c$ with $\| c\| =1$, we have that $c^{T} T_{n}$ has two terms, where $\BV \bigl( - c^{T} \Gamma_{0 n }^{-1/2}  \sum_{i = 1}^{ n} \psi_{0 i} \bigr) = 1$ and
\ban 
\BE 
\bigl[ p^{1/2} c^{T} \Gamma_{0n }^{-1} \Gamma_{1 n } R_{n} 
\bigr]^{2} = O(1),
\ean
using \eqref{eq:Trace}. Using \eqref{eq:CLTConditions} we also have that for any $c$ with $\|c\| = 1$, $ c^{T} T_{n}$ converges in distribution to $N (0, 1)$. This completes the proof.
\end{proof}

We now have a parallel result on consistency of the GBS resampling scheme. The essence of this theorem is that under the same set of conditions, several resampling schemes are consistent resampling procedures to implement the e-values framework. 

\begin{theorem}
\label{Theorem:ResamplingConsistency_Smooth}
Assume conditions (C1)-(C5). Additionally, assume that the resampling weights $\BW_{r n i}$   are exchangeable random variables satisfying the conditions
\eqref{eq:W_cond}. Define $\hat{B}_{n} = \mu_{n} \tau_{n}^{-1} \hat{\Gamma}^{1/2}_{0 n} \hat{\Gamma}^{-1}_{1 n}$, where $\hat{\Gamma}_{0 n}$ and $\hat{\Gamma}_{1 n}$ are sample equivalents of $\Gamma_{0 n}$ 
and $\Gamma_{1 n}$, respectively. Conditional on the data, 
$\hat{B}_{n} (\hat{\vectheta}_{r *} - \hat{\vectheta}_{*})$ converges weakly to the $p$-dimensional standard Gaussian distribution in probability. 

Under the additional condition (C6), defining 
$b_{n}  =\mu_{n} \tau_{n}^{-1} a_{n}$, 
the distributions of $a_{n} (\hat{\vectheta}_{*} - {\vectheta}_{*})$ 
and $b_{n} (\hat{\vectheta}_{r *} - \hat{\vectheta}_{*})$
converge to the same weak limit in probability.
\end{theorem}

\begin{proof}[Proof of Theorem~\ref{Theorem:ResamplingConsistency_Smooth}]
This proof has steps similar to that of the proof of Theorem~\ref{Theorem:CLT}, apart from several additional technicalities. We omit the details. 
\end{proof}

\section{Theoretical proofs}
The results in Section~\ref{sec:Theory} hold under more general conditions than in the main paper--specifically, when the asymptotic distribution of $\hat \beta_*$ comes from a general class of \textit{elliptic} distributions, rather than simply being multivariate Gaussian.

Following \citet{FangEtalBook}, the density function of an elliptically distributed random variable $\cE (\mu, \Sigma, g)$ takes the form:
$ h(x; \mu, \Sigma) = |\Sigma|^{-1/2} g ((x - \mu)^T \Sigma^{-1} (x - \mu)) $
where $\mu \in \BR^p$, $\Sigma \in \BR^{p \times p}$ is positive semi-definite, and $g$ is a density function that is non-negative, scalar-valued, continuous and strictly increasing. For the asymptotic distribution of $\hat\beta_*$, we assume the following conditions:

\begin{enumerate}[nolistsep]
\item[(A1)]
There exist a sequence of positive reals $a_n \uparrow \infty$, positive-definite (PD) matrix $\bfV \in \BR^{p \times p}$ and density $g$ such that $a_{n} (\hat \bftheta_{*} - \bftheta_{0} )$ converges to $\cE ( 0_p, \bfV, g)$ in distribution, denoted by $a_{n} (\hat \bftheta_{*} - \bftheta_{0} ) \leadsto \cE ( 0_p, \bfV, g)$;

\item[(A2)]
For almost every dataset $\cZ_n$, There exist PD matrices $\bfV_n \in \BR^{p \times p}$ such that  $\text{plim}_{n \raro \infty} \bfV_n = \bfV$.
\end{enumerate}

These conditions are naturally satisfied for a Gaussian limiting distribution.

\subsection{Proofs of main results}
\begin{proof}[Proof of Theorem \ref{Theorem:ThmRightNested}]
We divide the proof into four parts:

\begin{enumerate}[nolistsep,leftmargin=*]
\item ordering of adequate model e-values,
\item convergence of all adequate model e-values to a common limit,
\item convergence of inadequate model e-values to 0,
\item comparison of adequate and inadequate model e-values,
\end{enumerate}

\paragraph{Proof of part 1.}
Since we are dealing with a finite sequence of nested models, it is enough to prove that $e_n(\cM_{1} ) > e_n(\cM_{2} )$ for large enough $n$, when both $\cM_{1 }$ and $\cM_{2}$ are adequate models and $\cM_{1 } \prec  \cM_{2 }$.

Suppose $\BT_0 = \cE (0_p, I_p, g)$. Affine invariance implies invariant to rotational transformations, and since the evaluation functions we consider decrease along any ray from the origin because of (B5), $E ( \vectheta, \BT_0)$ is a monotonocally decreasing function of $ \|\vectheta \|$ for any $\vectheta \in \BR^p$. Now consider the models $\cM^0_{1}, \cM^0_{2}$ that have 0 in all indices outside $\cS_{1}$ and $\cS_{2}$, respectively. Take some $\vectheta_{10} \in \matTheta^0_{1}$, which is the parameter space corresponding to $\cM^0_{1}$, and replace its (zero) entries at indices $j \in \cS_{2} \setminus \cS_{1}$ by some non-zero $\bfdelta \in \BR^{p - | \cS_{2} \setminus \cS_{1} |}$. Denote it by $\vectheta_{1 \bfdelta}$. Then we shall have
\begin{align*}
\bftheta_{1 \bfdelta}^T \bftheta_{1 \bfdelta} > \bftheta_{10}^T \bftheta_{10}
\quad \Rightarrow \quad
D ( \bftheta_{10}, \BT_0) > D ( \bftheta_{1 \bfdelta}, \BT_0)
\quad \Rightarrow \quad
\BE_{s1} D ( \bftheta_{10}, \BT_0) > \BE_{s1} D ( \bftheta_{1 \bfdelta}, \BT_0),
\end{align*}
where $\BE_{s1}$ denotes the expectation taken over the marginal of the distributional argument $\BT_0$ at indices $\cS_{1}$. Notice now that by construction $\bftheta_{1 \bfdelta} \in \bfTheta^0_{2}$, the parameter space corresponding to $\cM^0_{2}$, and since the above holds for all possible $\bfdelta$, we can take expectation 
over indices $\cS_{2} \setminus \cS_{1}$ in both sides to obtain $\BE_{s1} D( \bftheta_{10},  \BT_0) > \BE_{s2} D ( \bftheta_{20}, \BT_0)$, with $\bftheta_{20}$ denoting a general element in $\bfTheta_{20}$.

Combining (A1) and (A2) we get $a_n \bfV_n^{-1/2} (\hat \bftheta_{*} - \bftheta_{0} ) \leadsto \BT_0$. 
Denote $\BT_n = [ a_n \bfV_n^{-1/2} (\hat \bftheta_{*} - \bftheta_{0} ) ]$, and choose a positive $\epsilon < (\BE_{s1} D ( \bftheta_{10}, \BT_0) - \BE_{s2} D( \bftheta_{20}, \BT_0))/2$. Then, for large enough $n$ we shall have
$$
\left| D ( \bftheta_{10}, \BT_n) - D ( \bftheta_{10}, \BT_0) \right| < \epsilon
\quad \Rightarrow \quad
| \BE_{s1} D ( \bftheta_{10}, \BT_n) - \BE_{s1} D ( \bftheta_{10}, \BT_0) | < \epsilon,
$$
following condition (B4). Similarly we have $| \BE_{s2} D( \bftheta_{20}, \BT_n) - \BE_{s2} D( \bftheta_{20}, \BT_0) | < \epsilon $ for the same $n$ for which the above holds. This implies $\BE_{s1} D( \bftheta_{10}, \BT_n) > \BE_{s2} D( \bftheta_{20}, \BT_n)$.

Now apply the affine transformation $\bft (\bftheta) = \bfV_n^{1/2} \bftheta / a_n + \bftheta_{0}$ to both arguments of the depth function above. This will keep the depths constant following affine invariance, i.e. $D(\bft ( \bftheta_{10}), [\hat \bftheta_{*}]) = D( \bftheta_{10}, \BT_n)$ and $D(\bft ( \bftheta_{20}), [\hat \bftheta_{*}]) = D( \bftheta_{20}, \BT_n)$. Since this transformation maps $\bfTheta^0_{1}$ to $\bfTheta_{1}$, the parameter space corresonding to $\cM_{1}$, we get $\BE_{s1} D(\bft ( \bftheta_{10}), [\hat \bftheta_{*}]) > \BE_{s2} D(\bft ( \bftheta_{20}), [\hat \bftheta_{*}])$, i.e. $e_n ( \cM_{1} ) > e_n ( \cM_{2} )$.

\paragraph{Proof of part 2.}
For the full model $\cM_*$, $a_n(\hat \theta_* - \theta_0) \leadsto \cE_p(0, V, g)$ by (A1). It follows now from a direct application of condition (B3) that $e_n (\cM_*) \rightarrow \BE(Y, [Y]$ where $Y \sim \cE(0_p, V, g)$.

For any other adequate model  $\cM$, we use ({\sc B1}) property of $D$:
\begin{align}
D ( \hat{\vectheta}_{m}, [ \hat{\vectheta}_{*} ]) = 
D \Bigl( \hat{\vectheta}_{m} - \vectheta_{0}, \bigl[ \hat{\vectheta}_{0} - \vectheta_{0} \bigr] \Bigr),
\end{align}
and decompose the first argument
\begin{align}\label{equation:ThmRightWrongProofPart2Eqn1}
\hat{\vectheta}_{m} - \vectheta_{0} = 
( \hat{\vectheta}_{m } - \hat{\vectheta}_{*} ) 
		+ ( \hat{\vectheta}_{*} - \vectheta_{0} ).
\end{align}
We now have 
\ban
\hat{\vectheta}_{m } &= \vectheta_{m } + a_{n}^{-1} T_{m n},\\
\hat{\vectheta}_{* } &= \vectheta_{* } + a_{n}^{-1} T_{* n},
\ean
where $T_{m n}$ is non-degenerate at the indices $\cS$, and $T_{* n} \leadsto \cE(0_p, V, g)$. For the first summand of the right-hand side in \eqref{equation:ThmRightWrongProofPart2Eqn1} we then have
\begin{align}\label{equation:ThmRightWrongProofPart2Eqn2}
\hat{\vectheta}_{m} - \hat{\vectheta}_{*} = \vectheta_{m} - \vectheta_{0} + R_{n},
\end{align}
where $\BE  \| R_n^2 \| = O ( a_{n}^{-2})$, and $\theta_{m j}$ equals $\theta_{0 j}$ in indices $j \in \cS$ and $C_{j}$ elsewhere. Since $\cM$ is adequate, $ \vectheta_{m } = \vectheta_{0}$. Thus, substituting the right-hand side in \eqref{equation:ThmRightWrongProofPart2Eqn1} we get
\begin{align}
\left| D \left( \hat{\vectheta}_{m } - \vectheta_{0}, 
\left[ \hat{\vectheta}_{*} - \vectheta_{0} \right] \right) 
- D \left( \hat{\vectheta}_{*} - \vectheta_{0}, 
\left[ \hat{\vectheta}_{*} - \vectheta_{0} \right] \right) \right|
\leq \| R_{n} \|^{\alpha},
\end{align}
from Lipschitz continuity of $D (\cdot)$ given in (B2). Part 2 now follows.

\paragraph{Proof of Part 3.} 
Since the depth function $D$ is invariant under location and scale transformations, we have
\begin{align}\label{equation:ThmRightWrongProofEqn1}
D ( \hat{\vectheta}_{m }, [ \hat{\vectheta}_{*} ]) = 
D \Bigl( a_{n} (\hat{\vectheta}_{m } - \vectheta_{0}), 
\bigl[ a_{n} (\hat{\vectheta}_{*} - \vectheta_{0}) 
\bigr] \Bigr).
\end{align}
Decomposing the first argument,
\begin{align}
a_{n} (\hat{\vectheta}_{m } - \vectheta_{0}) = 
a_{n} (\hat{\vectheta}_{m } - \vectheta_{m }) + a_{n} (\vectheta_{m } - \vectheta_{0}). 
\end{align}
Since $\cM$ is inadequate, $\sum_{j \notin \cS} |C_j - \theta_{0 j} | > 0$, i.e. $\theta_m$ and $\theta_0$ are not equal in at least one (fixed) index. Consequently as $n \rightarrow \infty$, $ \| a_{n}(\vectheta_{m } - \vectheta_{0}) \| \rightarrow \infty$, thus $e_n (\cM) \rightarrow 0$ by condition (B4).

\paragraph{Proof of part 4.}

For any inadequate model $\cM_{j}, k < j \leq K$, suppose $N_{j}$ is the integer such that $e_{n_1} ( \cM_{j_1}) < e_{n_1} (\cM_{* })$ for all $n_1 > N_{j}$. Part 3 above ensures that such an integer exists for every inadequate model. Now define $N = \max_{k < j \leq K} N_{j}$. Thus $e_{n_1} ( \cM_{* } )$ is larger than e-values of all inadequate models $\cM_{j_1}$ for $k < j \leq K$.
\end{proof}

\begin{proof}[Proof of Corollary~\ref{Corollary:ZeroModelCorollary}]
By construction, $\cM_0$ is nested in all other adequate models in $\BM_0$. Hence Theorem 4.1 implies $e_n ( \cM_0) > e_n (\cM^{ad}) > e_n (\cM^{inad})$ for any adequate model $\cM^{ad}$ and inadequate model $\cM^{inad}$ in $\mathbb M_0$ and large enough $n$.
\end{proof}

\begin{proof}[Proof of Corollary~\ref{Corollary:AlgoCorollary}]
Consider $j \in \cS_{0}$. Then $\vectheta_{0} \notin \cM_{-j}$, hence $\cM_{-j}$ is inadequate. By choice of $n_1$ from Corollary~4.1, $e$-values of all inadequate models are less than 
that of $\cM_{*}$, hence $e_{n_1} (\cM_{-j}) < e_{n_1} (\cM_{*})$.

On the other hand, suppose there exists a $j$ such that $e_{n_1} (\cM_{-j}) \leq e_{n_1} (\cM_*)$ but $j \notin \cS_{0}$. Now $j \notin \cS_{0}$ means that $\cM_{-j}$ is an adequate model. Since $\cM_{-j} $ is nested within $ \cM_{*}$ for any $j$, and the full model is always adequate, we have $e_{n_1} (\cM_{-j}) > e_{n_1} (\cM_{*})$ by 
Theorem~4.1: leading to a contradiction and thus completing the proof.
\end{proof}

\begin{proof}[Proof of Theorem~\ref{thm:BootConsistency}]
Corollary~4.2 implies that
$$
\cS_0 = \{j: e_n (\cM_{-j}) < e_n (\cM_*) \}.
$$
Now define $ \bar \cS_0 = \{j: e_{r n} (\cM_{-j}) < e_{r n} (\cM_*) \} $. We also use an approximation result.

\begin{lemma}\label{lemma: approx_evalue}
For any adequate model $\cM$, the following holds for fixed $n$ and an exchangeable array of GB resamples $\BW_r$ as in the main paper:
\baq
e_{rn} = e_n (\cM) + R_r, \quad \BE_r | R_{r} |^2 = o_P(1).
\eaq
\end{lemma}

Using Lemma~\ref{lemma: approx_evalue} for $\cM_{-j}$ and $\cM_*$ we now have 
\begin{align*}
e_{r n} (\cM_{-j}) &= e_{n} (\cM_{-j}) + R_{rj},\\
e_{r n} (\cM_{*}) &= e_{n} (\cM_{*}) + R_{r*},
\end{align*}
such that $\BE_r | R_{r*} |^2 = o_P(1)$ and $\BE_r | R_{rj} |^2 = o_P(1)$ for all $j$. Hence $P_1( \bar \cS_0 = \cS_0) \rightarrow 1$ as $n \rightarrow \infty$, $P_1$ being probability conditional on the data. Similarly one can prove that the probability conditional on the bootstrap samples that $\bar \cS_0 = \hat \cS_0$ holds goes to 1 as $R, R_1 \rightarrow \infty$, completing the proof.
\end{proof}

\subsection{Proofs of auxiliary results}
\begin{proof}[Proof of Lemma~\ref{lemma: approx_evalue}]

We decompose the resampled depth function as
\ban 
D \Bigl( \hat{\vectheta}_{ r_1 m}, [ \hat{\vectheta}_{ r *} ] \Bigr)
& = D \Bigl( \frac{a_n}{\tau_{n}} \bigl( \hat{\vectheta}_{r_1 m} - \hat{\vectheta}_{*} \bigr), 
 		\left[ \frac{a_n}{\tau_{n}} \bigl( \hat{\vectheta}_{r *} - \hat{\vectheta}_{*} \bigr) \right] \Bigr) \\
& = D \Bigl( 
 \frac{a_n}{\tau_{n}} \bigl( \hat{\vectheta}_{r_1 m} - \hat{\vectheta}_{ m} \bigr)
 - \frac{a_n}{\tau_{n}} \bigl( \hat{\vectheta}_{ m} - \hat{\vectheta}_{*} \bigr), 
 		\left[ \frac{a_n}{\tau_{n}} \bigl( \hat{\vectheta}_{r *} - \hat{\vectheta}_{*} \bigr) \right] \Bigr).
\ean
Conditional on the data, $(a_n / \tau_n) (\hat \theta_{r_1 m} - \hat \theta_{m})$ has the same weak limit as $a_n (\hat \theta_{m} - \theta_{m})$, and the same holds for $(a_n / \tau_n) (\hat \theta_{r_1 *} - \hat \theta_{*})$ and $a_n (\hat \theta_{*} - \theta_{*})$. Now \eqref{equation:ThmRightWrongProofPart2Eqn2} and $\tau_n \rightarrow \infty$ combine to give
$$
\frac{a_n}{\tau_{n}} \bigl( \hat{\vectheta}_{ m} - \hat{\vectheta}_{*} \bigr) \stackrel{P}{\rightarrow} 0,
$$
as $n \rightarrow \infty$. Lemma~\ref{lemma: approx_evalue} follows directly now.
\end{proof}

\section{Details of experiments}
\label{sec:exp}
Among the competing methods, for stepwise regression there is no tuning. MIO requires specification of a range of desired sparsity levels and a time limit for the MIO solver to run for each sparsity level in the beginning. We specify the sparsity range to be $\{1, 3, \ldots, 29 \}$ in all settings to cover the sparsity levels across different simulation settings, and the time limit to be 10 seconds. We select the optimal MIO sparsity level as the one for which the resulting estimate gives the lowest BIC. We use BIC to select the optimal regularization parameter for Lasso and SCAD as well. The Knockoff filters come in two flavors: Knockoff and Knockoff+. We found that Knockoff+ hardly selects any features in our settings, so use the Knockoffs for evaluation, setting its false discovery rate threshold at the default value of 0.1. We shall include these details in the appendix.


\section{Details of Indian Monsoon data}
\label{Section:IndianMonsoon}
Various  studies indicate that our knowledge about the physical drivers of precipitation in India is incomplete; this is in addition to the known difficulties in modeling precipitation itself \citep{KnuttiEtal10, TrenberthEtal03, Trenberth11, WangEtal05}. For example, \cite{Goswami05} discovered an upward trend in frequency and magnitude of extreme rain events, using daily central Indian rainfall data on a $10^{\circ}$ $\times$ $12^{\circ}$ grid, but a similar study on a $1^{\circ}\times 1^{\circ}$ gridded data by \cite{GhoshEtal16} suggested that there are both increasing and decreasing trends of extreme rainfall events, depending on the location. Additionally, \cite{Krishnamurthy09} reported increasing trends in exceedances of the 99th percentile of daily rainfall; however, there is also a decreasing trend for exceedances of the 90th percentile data in many parts of India. Significant spatial and temporal variabilities at various scales have also been discovered for Indian Monsoon \citep{Dietz2014, Dietz2015Chapter}.

We attempt to identify the driving factors behind precipitation during the Indian monsoon season using our e-value based feature selection technique. Data is obtained from the repositories of the National Climatic Data Center (NCDC) and National Oceanic and Atmospheric Administration (NOAA), for the years 1978-2012. 
We obtained data 35 potential predictors of the Indian summer precipitation:

\noindent\textbf{(A) Station-specific}: (from 36 weather stations across India) Latitude, longitude, elevation, maximum and minimum temperature, tropospheric temperature difference ($\Delta TT$), Indian Dipole Mode Index (DMI), Ni\~{n}o 3.4 anomaly;

\noindent\textbf{(B) Global}:
\begin{itemize}[nolistsep,leftmargin=*]
\item $u$-wind and $v$-wind at 200, 600 and 850 mb;

\item Ten indices of Madden-Julian Oscillations: 20E, 70E, 80E, 100E, 120E, 140E, 160E, 120W, 40W, 10W;

\item Teleconnections: North Atlantic Oscillation (NAO), East Atlantic (EA), West Pacific (WP), East Pacific/North Pacific (EPNP), Pacific/North American (PNA), East Atlantic/Western Russia (EAWR), Scandinavia (SCA), Tropical/Northern Hemisphere (TNH), Polar/Eurasia (POL);

\item Solar Flux;

\item Land-Ocean Temperature Anomaly (TA).
\end{itemize}

These covariates are all based on existing knowledge and conjectures from the actual Physics driving Indian summer precipitations. The references provided earlier in this section, and multiple references contained therein may be used for background knowledge on the physical processes related to Indian monsoon rainfall, which after decades of study remains one of the most challenging problems in climate science.

\section{Details of fMRI data implementation}
\label{sec:fmri}

\begin{figure}
\centering
\includegraphics[width=.3\textwidth]{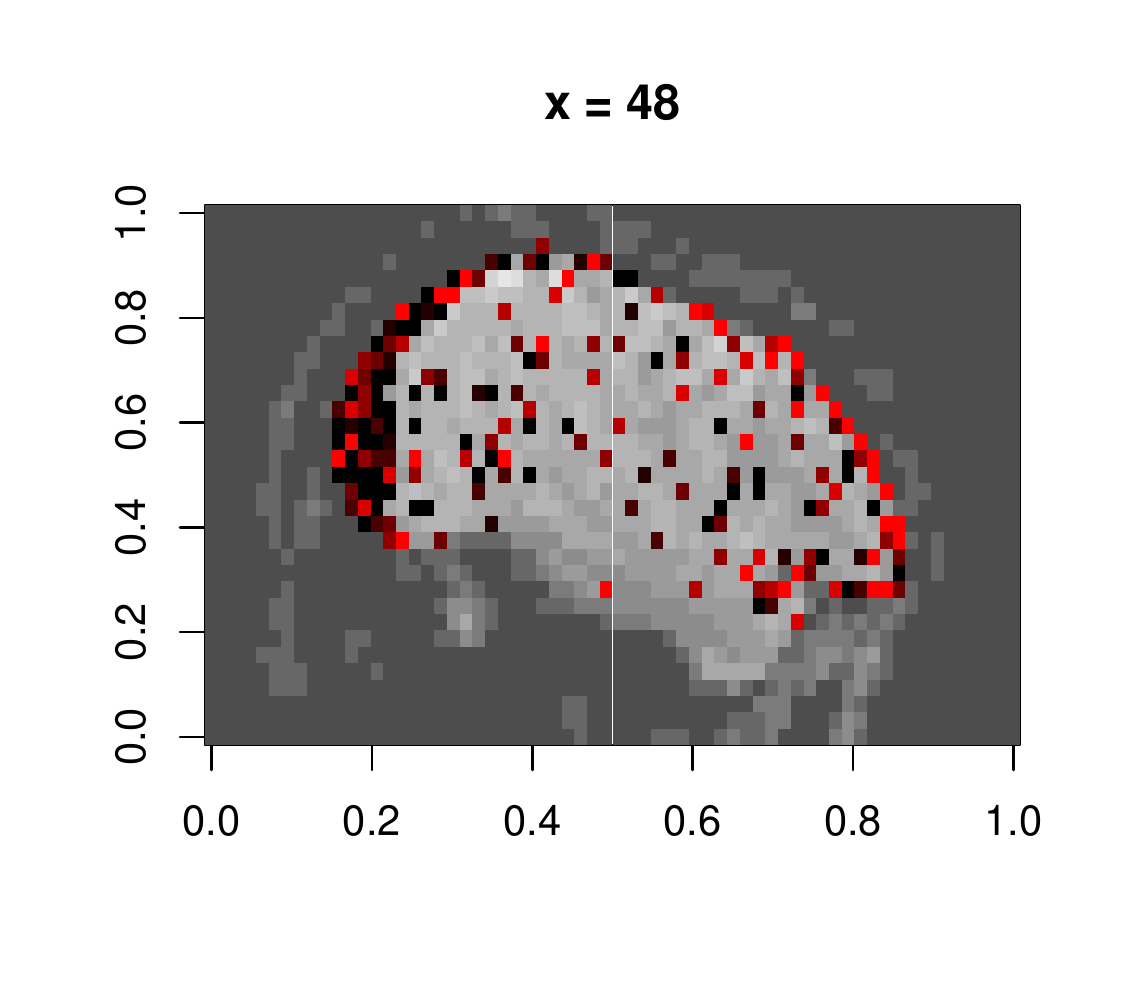}
\includegraphics[width=.3\textwidth]{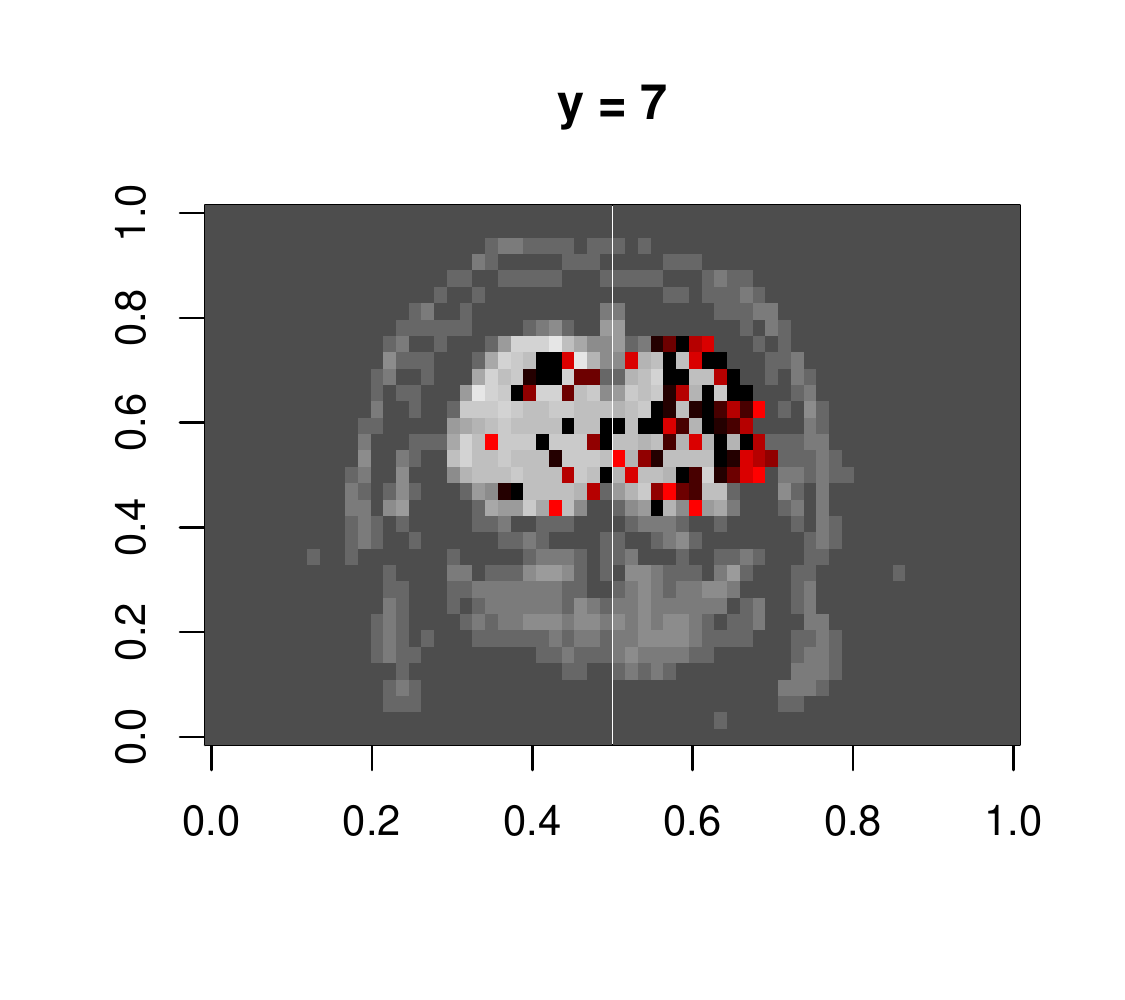}
\includegraphics[width=.3\textwidth]{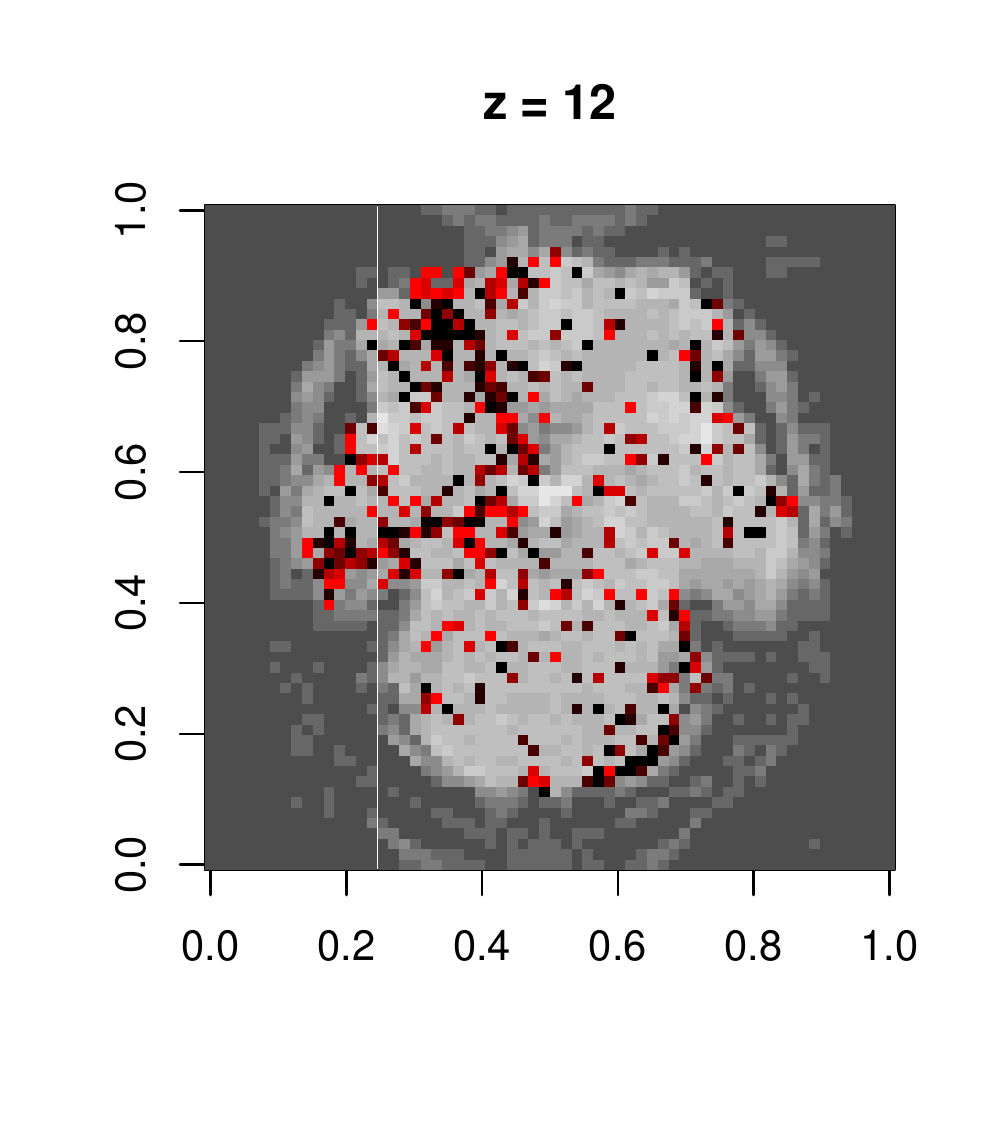}
\caption{(Top) Plot of significant $p$-values at 95\% confidence level at the specified cross-sections; (bottom) a smoothed surface obtained from the $p$-values clearly shows high spatial dependence in right optic nerve, auditory nerves, auditory cortex and left visual cortex areas}
\label{fig:fmrifigure}
\end{figure}

Typically, the brain is divided by a grid into three-dimensional array elements called voxels, and activity is measured at each voxel. More specifically, a series of three-dimensional images are obtained by measuring Blood Oxygen Level Dependent (BOLD) signals for a time interval as the subject performs several tasks at specific time points. A single fMRI image typically consists of voxels in the order of $10^5$, which makes even fitting the simplest of statistical models computationally intensive when it is repeated for all voxels to generate inference, e.g. investigating the differential activation of brain region in response to a task.

The dataset we work with comes from a recent study involving 19 test subjects and two types of visual tasks \citep{WakemanHenson15}. Each subject went through 9 runs, in which they were showed faces or scrambled faces at specific time points. In each run 210 images were recorded in 2 second intervals, and each 3D image was of the dimension of $64 \times 64 \times 33$, which means there were 135,168 voxels. Here we use the data from a single run on subject 1, and perform a voxelwise analysis to find out the effect of time lags and BOLD responses at neighboring voxels on the BOLD response at a voxel. Formally we consider separate models at voxel $i \in \{1,2,...,V\}$, with observations across time points $t \in \{1,2,...,T\}$.

Clubbing together the stimuli, drift, neighbor and autoregressive terms into a combined design matrix $\tilde \bfX = (\tilde \bfx(1)^T,...,\tilde \bfx(T)^T)^T$ and coefficient vector $\bftheta_i$, we can write $y_i(t) = \tilde \bfx(t)^T \bftheta_i + \epsilon_i(t) $. We estimate the set of non-zero coefficients in $\bftheta_i$ using the e-value method. Suppose this set is $R_i$, and its subsets containing coefficient corresponding to neighbor and non-neighbor (i.e. stimuli and drift) terms are $S_i$ and $T_i$, respectively. To quantify the effect of neighbors we now calculate the corresponding $F$-statistic:
$$
F_i = \frac{(\sum_{n \in S_i} \tilde x_{i,n} \hat\theta_{i,n})^2}{(y_i(t) - \sum_{n \in T_i} \tilde x_{i,n} \hat\theta_{i,n})^2} \frac{|n-T_i|}{|S_i|},
$$
and obtain its $p$-value, i.e. $P(F_i \geq F_{|S_i|,|n-T_i|})$.

Figure~\ref{fig:fmrifigure} shows plots of the voxels with a significant $p$-value from the above $F$-test, with a darker color associated with lower p-value, as opposed to the smoothed surface in the main paper. Most of the significant terms were due to the coefficients corresponding to neighboring terms. A very small proportion of voxels had any autoregressive effects selected (less than 1\%), and most of them were in regions of the image that were outside the brain, indicating noise.

In future work, we aim to expand on the encouraging findings and repeat the procedure on other individuals in the study. An interesting direction here is to include subject-specific random effects and correlate their clinical outcomes (if any) to observed spatial dependency patterns in their brain.

\end{document}